\documentclass[10pt,journal,compsoc]{IEEEtran}

\ifCLASSOPTIONcompsoc
  \usepackage[nocompress]{cite}
\else
  \usepackage{cite}
\fi

\ifCLASSINFOpdf
\else
\fi

\usepackage{hyperref}       
\hypersetup{hidelinks}
\usepackage{url}            
\usepackage{booktabs}       
\usepackage{amsfonts}       
\usepackage{nicefrac}       
\usepackage{microtype}      
\usepackage{xcolor}         

\usepackage{amssymb}
\usepackage{amsmath,mathrsfs,dsfont}
\usepackage{nicefrac}
\usepackage{algorithmic}
\usepackage{algorithm}

\allowdisplaybreaks
\usepackage{booktabs}

\usepackage{color}
\usepackage{enumitem}

\usepackage{array,cite}
\usepackage{graphicx,tikz}
\usepackage[mathscr]{euscript}
\usepackage{amsthm}
\usepackage{cite}

\usepackage{bm}
\usepackage{bbm}
\usepackage{color}

\usepackage{color}
\usepackage{epstopdf}
\usepackage{subcaption}
\usepackage{cleveref}
\usepackage{thmtools}
\usepackage{thm-restate}

\usepackage{bbding}
\usepackage{mathtools}
\allowdisplaybreaks
\usepackage{multirow}
\usepackage{adjustbox}
\usepackage{slashbox}
\usepackage{wrapfig}
\DeclareUnicodeCharacter{FB01}{fi}
\DeclareUnicodeCharacter{FB02}{fi}

\graphicspath{{illustrations/}}

\newcounter{optproblem}

\theoremstyle{plain}
\newtheorem{theorem}{Theorem}[section]
\newtheorem{proposition}[theorem]{Proposition}
\newtheorem*{lemma*}{Lemma}

\DeclareMathAlphabet{\pazocal}{OMS}{zplm}{m}{n}
\DeclareMathAlphabet{\mathpzc}{OMS}{pzc}{m}{it}

\setlist[itemize]{leftmargin=*}

     \def\RR{\mathbb{R}}

 \def\cB{{\cal  B}}
 \def\cC{{\cal  C}}

 \def\cW{{\cal  W}}

\def\+#1{\mathcal{#1}}
\def\-#1{\textup{#1}}

\def\set#1{\left\{ #1 \right\}}
\def\pth#1{\left( #1 \right)}

\def\defeq {\coloneqq}

\newcommand{\La}{\left\langle\kern-0.64ex\left\langle}
\newcommand{\Ra}{\right\rangle\kern-0.64ex\right\rangle}

\def\Norm#1#2{{\left\vert\kern-0.4ex\left\vert\kern-0.4ex\left\vert #1
    \right\vert\kern-0.4ex\right\vert\kern-0.4ex\right\vert}_{#2}}
\def\norm#1#2{{\left\|#1\right\|}_{#2}}

\def\ltwonorm#1{\norm{#1}{2}}

\newcommand{\1}{{\rm 1}\kern-0.25em{\rm I}}
\def\indict#1{{\rm 1}\kern-0.25em{\rm I}_{\set{#1}}}

\def\set#1{\left\{#1\right\}}

\def \Pr {\textup{Pr}}
\newcommand{\Prob}[1]{\Pr\left[#1\right]}

\newcommand{\beq}{\begin{equation}}
\newcommand{\eeq}{\end{equation}}
\newcommand{\beqa}{\begin{eqnarray}}
\newcommand{\eeqa}{\end{eqnarray}}
\newcommand{\beqas}{\begin{eqnarray*}}
\newcommand{\eeqas}{\end{eqnarray*}}
\def\bal#1\eal{\begin{align}#1\end{align}}
\def\bals#1\eals{\begin{align*}#1\end{align*}}
\def\bsal#1\esal{\begin{small}\begin{align}#1\end{align}\end{small}}
\def\bsals#1\esals{\begin{small}\begin{align*}#1\end{align*}\end{small}}
\def\bsfal#1\esfal{\begin{small}\begin{flalign}#1\end{flalign}\end{small}}

\begin{document}

%
\title{Efficient Visual Transformer by Learnable Token Merging}
%
%
%
%

\author{Yancheng~Wang and
        Yingzhen~Yang
\IEEEcompsocitemizethanks{\IEEEcompsocthanksitem Yancheng~Wang and Yingzhen Yang are with School of Computing and
Augmented Intelligence, Arizona State University, Tempe, AZ, 85281.\protect\\
E-mail: ywan1053@asu.edu, yingzhen.yang@asu.edu
}
}

\IEEEtitleabstractindextext{%
\begin{abstract}
Self-attention and transformers have been widely used in deep learning. Recent efforts have been devoted to incorporating transformer blocks into different neural architectures, including those with convolutions, leading to various visual transformers for computer vision tasks. In this paper, we propose a novel and compact transformer block, Transformer with Learnable Token Merging (LTM), or LTM-Transformer. LTM-Transformer performs token merging in a learnable scheme. LTM-Transformer is compatible with many popular and compact transformer networks, and it reduces the FLOPs and the inference time of the visual transformers while maintaining or even improving the prediction accuracy. In the experiments, we replace all the transformer blocks in popular visual transformers, including MobileViT, EfficientViT, ViT, and Swin, with LTM-Transformer blocks, leading to LTM-Transformer networks with different backbones. The LTM-Transformer is motivated by reduction of Information Bottleneck, and a novel and separable variational upper bound for the IB loss is derived. The architecture of the mask module in our LTM blocks, which generates the token merging mask, is designed to reduce the derived upper bound for the IB loss. Extensive results on computer vision tasks evidence that LTM-Transformer renders compact and efficient visual transformers with comparable or much better prediction accuracy than the original visual transformers. The code of the LTM-Transformer is available at \url{https://github.com/Statistical-Deep-Learning/LTM}.
\end{abstract}

}

\maketitle
\IEEEdisplaynontitleabstractindextext
\IEEEpeerreviewmaketitle
\IEEEraisesectionheading{\section{Introduction}\label{sec:introduction}}

\IEEEPARstart{B}{uilding} upon the success of Transformer in natural language processing ~\cite{vaswani2017attention}, visual transformers have demonstrated remarkable performance across a wide range of tasks~\cite{yuan2021tokens, dosovitskiy2020image, liu2021swin, zhu2020deformable, cai2022efficientvit, wang2024visual, wang2024learning}. However, the achievements of visual transformers are accompanied with heavy computational costs~\cite{dosovitskiy2020image, touvron2021training}, making their deployment impractical under resource-limited scenarios. The aforementioned limitations have spurred recent research endeavors aimed at developing efficient visual transformers.
In this paper, we study the problem of accelerating visual transformers by token merging.

Token merging is an effective method for reducing the FLOPs and improving the inference speed of visual transformers~\cite{han2015learning, zhou2020rethinking, sun2021cascade, kim2024token, bonnaerens2023learned, ToMe}. However, most existing token merging methods~\cite{rao2021dynamicvit, ToMe, kim2024token, bonnaerens2023learned} largely sacrifice the prediction accuracy of the original transformer networks for reduced computation costs~\cite{ToMe, bello2019attention}. These methods~\cite{kim2024token, ToMe} generally focus on identifying and merging similar tokens by averaging their features. However, such merging strategies, which are based solely on feature similarity, can potentially diminish the informative features in the tokens that are critical to the prediction tasks. Therefore, it remains an interesting and important question that if we can perform token merging while preserving a compelling performance of the visual transformers after token merging. To this end, we propose a novel transformer block, Transformer with Learnable Token Merging,
or LTM-Transformer, which learns how to merge tokens while exhibiting a compelling generalization capability of the transformer with merged tokens.

\noindent\textbf{Motivation.}
Due to the fact that the FLOPs of a visual transformer largely depend on the number of tokens in all the transformer blocks, the FLOPs of a visual transformer can be significantly reduced by reducing the number of tokens in all the transformer blocks. As reviewed in Section~\ref{sec:efficient-transformers-related-works}, the existing token merging methods do not have a principled way of merging original tokens based on their informativeness and importance for classification tasks.
\textbf{Our goal is to design a principled informative token merging method to merge the output tokens of every transformer block into fewer tokens without largely sacrificing the prediction accuracy of the original visual transformer, and more informative original tokens contribute more to the merged tokens in the token merging process.} However, directly merging the output tokens, even by carefully designed methods~\cite{kim2024token, bonnaerens2023learned, ToMe}, would adversely affect the performance of the model. In this paper, we propose to maintain a compelling prediction accuracy of a visual transformer with token merging by an informative token merging process. In our LTM-Transformer block, the original attention output tokens of a transformer block are merged into less target tokens, and every target token is an informative weighted average of the original output tokens. All the target tokens, or the merged tokens, are the final attention output tokens for the LTM-Transformer block, which are fed to an MLP to produce the output of the LTM-Transformer block as illustrated by Figure~\ref{fig:ltm}.

Such a token merging process in LTM-Transformer is primarily inspired by the well-known presence of considerable redundancy in the original output tokens of transformer blocks~\cite{rao2021dynamicvit, ToMe}. As different tokens have varying importance in modeling the visual features at a particular transformer block, it is natural to compute an informative aggregation of the original attention output tokens as the final (target) attention output tokens so that more informative and more important tokens contribute more to the merged tokens with a larger weight in the weighted average in the aggregation process.

Such an idea of informative token merging can also be viewed from the perspective of Information Bottleneck (IB).
Let $X$ be the input image. Let $Z$ be the original attention output tokens, which are merged into the merged tokens denoted by $\tilde X$, and let $Y$ be the ground truth training labels for a classification task. $\tilde X$ has fewer tokens than $Z$. The principle of IB is to maximize the mutual information between $\tilde X$ and $Y$ while minimizing the mutual information between $\tilde X$ and $X$. That is, IB encourages the network to learn the merged tokens more correlated with the class labels while reducing their correlation with the input. Extensive empirical and theoretical works have evidenced that models respecting the IB principle enjoy compelling generalization. With the informative token merging process in LTM-Transformer, the merged tokens $\tilde X$ are the informative aggregation of the original attention output tokens $Z$,  so $\tilde X$ are less correlated with the training images and in this manner the IB principle is better adhered. This is reflected in Table~\ref{tab:ablation-IB-loss_finetune} and Table~\ref{tab:ib_loss_ablation_appendix} in Section~\ref{sec:ablation-study_IB}, where a model for ablation study with token merging, ToMe, enjoys less IB loss than the vanilla transformer, MobileViT-S. This observation indicates that the IB principle is better respected by the
token merging process in ToMe. In order to further decrease the IB loss, we propose an Information Bottleneck (IB) inspired token merging process, where a LTM-Transformer block generates an informative token merging mask, which reduces the IB loss for visual transformers. For example, our model termed ``LTM-MobileViT-S'' in Table~\ref{tab:ablation-IB-loss_finetune} and Table~\ref{tab:ib_loss_ablation_appendix} is the visual transformer with the IB loss reduced by replacing all the transformer blocks in MobileViT-S with LTM-Transformer blocks so that more informative merged tokens are generated by the proposed informative token merging process.
While ToMe hurts the prediction accuracy compared to the vanilla model,
our LTM-Transformer enjoys even higher top-1 accuracy than
the vanilla MobileViT-S, and we have the same observations for MobileViT-XS
and EfficientViT.

\noindent\textbf{Contributions. }
The contributions of this paper are presented as follows.

First, we present a novel and compact transformer block termed Transformer with Information Bottleneck inspired Token Merging, or LTM. Our LTM block
generates an informative token merging mask which reduces the IB loss. The LTM blocks can be used to replace all the transformer blocks in many popular vision transformers, rendering compact vision transformers with competitive performance. The effectiveness of LTM is evidenced by replacing all the transformer blocks in popular vision transformers, including MobileViT~\cite{mobilevit},  EfficientViT~\cite{cai2022efficientvit}, ViT~\cite{dosovitskiy2020image}, and Swin~\cite{liu2021swin}, with LTM blocks,  for image classification, object detection and instance segmentation tasks.

Second, we propose an informative token merging process for vision transformers, which can reduce the IB loss. As a first step, we derive a
novel and \textit{separable} variational upper bound for the IB loss associated with token merging, which is $I(\tilde X(G), X) - I(\tilde X(G),Y)$ where
$I(\cdot,\cdot)$ denotes mutual information and $G$ is the token merging mask in LTM.
We then view a transformer with multiple LTM blocks as an iterative process for the reduction of the IB loss by gradient descent,
and every LTM block simulates one-step gradient descent on the variational upper bound for the IB loss.
Inspired by this understanding, the token merging mask at the current layer is generated from
the token merging mask at the previous layer and the input tokens at the current layer by a learnable mask module, following the formula of gradient descent as in (\ref{eq:Gmask-GD-MLP-SGD}) in Section~\ref{sec:LTM-token-mask-optimization}.
As a result, such informative token merging process generated in a network with LTM blocks
enjoys reduced IB loss, which is evidenced in our ablation study in Section~\ref{sec:ablation-study_IB}.
Due to the separability of the variational upper bound for the IB loss, a neural network with LTM blocks can be trained in an end-to-end manner with standard SGD.

\noindent \textbf{Visualization for Informative Token Merging.}
Figure~\ref{fig:merge-weights} illustrates the informative token merging process by LTM with a comparison to LTMP. We remark that LTMP uses uniform merging weights, so their token merging process would not weight informative tokens property, resulting in their inferior performance in classification.

It is worthwhile to mention that our LTM models can be either fine-tuned from pre-trained backbones or trained from scratch. As evidenced in Table~\ref{tab:finetune}, our LTM models always outperform the current state-of-the-art token merging methods, including the fine-tuning-based method LTMP~\cite{bonnaerens2023learned}, when fine-tuned for the same number of epochs.
We remark that as shown in Table~\ref{tab:ablation-IB-loss_finetune}, the baseline token merging method, ToMe, and LTMP, can already reduce the IB loss. By replacing all the transformer blocks with our LTM blocks, the networks with LTM exhibit even smaller IB loss and enjoy higher classification accuracy and less FLOPs, either trained from scratch or fine-tuned from pre-trained models. Furthermore, as shown in Table~\ref{tab:imagenet_results}, our LTM models also outperform all the competing token merging methods when trained from scratch.  Importantly, extensive experiment results on various computer vision tasks demonstrate the compelling performance of LTM networks compared to the competing baselines.


This paper is organized as follows. The related works in efficient vision transformers and compression of vision transformers by pruning or token merging are discussed in Section~\ref{sec:related-works}. The formulation of LTM is detailed in Section~\ref{sec:formulation}. The effectiveness of LTM is demonstrated in Section~\ref{sec:experiments} for image classification, object detection and instance segmentation tasks, by replacing all the transformer blocks of various popular vision transformers, including MobileViT~\cite{mobilevit},  EfficientViT~\cite{cai2022efficientvit}, ViT~\cite{dosovitskiy2020image}, and Swin~\cite{liu2021swin}, with LTM blocks.
We conclude the paper in Section~\ref{sec:conclusion}.
We use $[n]$ to denote natural numbers between $1$
and $n$ inclusively.

\section{Related Works}\label{sec:related-works}
\subsection{Efficient Visual Transformers}
\label{sec:efficient-transformers-related-works}
Visual transformer models have recently achieved superior performance on a variety of computer vision applications~\cite{dosovitskiy2020image, liu2021swin, carion2020end, zhu2020deformable, wang2022adaptive}.
However, visual transformers often encounter high computational demands due to the quadratic complexity of the point-wise attention and numerous Multi-Layer Perceptron (MLP) layers.
To mitigate the challenges of high computational costs, various strategies have been developed \cite{zhu2020deformable, yuan2021tokens}, primarily aimed at refining the network architectures and incorporating sparse mechanisms for efficient computation. These include the integration of convolutions into transformer networks~\cite{mobilevit, cai2022efficientvit, liu2023efficientvit}, the use of knowledge distillation for training more efficient transformers~\cite{graham2021levit, radosavovic2020designing, gong2022nasvit}, and compressing existing visual transformers with methods such as pruning~\cite{SViTE, WDPruning, kong2022spvit}.
Techniques for compressing visual transformers generally fall into three categories: (1) Channel Pruning, which targets the elimination of superfluous heads and channels within ViT blocks~\cite{SViTE, ViT-Slim, zheng2022savit}; (2) Block Pruning, which involves removing redundant transformer blocks~\cite{UVC, WDPruning}; (3) Token Pruning and Token Merging, which prune less important tokens and merge similar tokens in the input of transformer blocks~\cite{rao2021dynamicvit, kong2022spvit, ToMe, VTC-LFC}.

In this paper, we focus on learning to merge tokens guided by the information bottleneck theory of deep learning and primarily review existing works on Token Pruning and Merging~\cite{VTC-LFC, rao2021dynamicvit, ToMe, bonnaerens2023learned, kim2024token}. DynamicViT~\cite{rao2021dynamicvit} observes that the prediction in visual transformers is only based on a subset of the most informative tokens and proposes a hierarchical token sparsification framework to prune redundant tokens.
ToMe~\cite{ToMe} proposes a graph-based matching algorithm that combines similar tokens in each visual transformer block of a pre-trained visual transformer.
LTMP~\cite{bonnaerens2023learned} learns threshold masking modules that dynamically determine which tokens to merge and prune in a unified framework similar to DynamicViT.
ToFu~\cite{kim2024token} also combines token pruning and token merging. Instead of averagely merging similar tokens, ToFu proposes a conventional average merging module to improve the quality of merged tokens.

\subsection{Related Works about Information Bottleneck}
\label{sec:IB-related-works}
\cite{saxe2019information} provides the first in-depth analysis of conventional information bottleneck (IB) theories and deep learning to establish the connection between the nonlinearity of neural networks and the compression phase of training.
Building on the theory of IB, \cite{lai2021information} proposes a probabilistic attention module reducing mutual information between the input and the masked representation while increasing mutual information between the masked representation and the task label.
Further exploring the mechanics of IB in deep learning, \cite{zhou2022understanding} finds that self-attention mechanisms can be interpreted as iterative steps in optimizing the IB objective, which explains the advantages of self-attention in learning robust representation.
Distinct from most existing methods that implicitly incorporate the IB principle, our work adopts a direct and innovative approach. We aim to optimize a
novel and separable variational upper bound of the IB loss with a learnable token merging method. The proposed LTM-Transformer leads to compelling performance on many popular visual transformer architectures with lower computation cost, benefiting from the learnable token merging mechanism guided by the IB principle.



\begin{figure*}[!htbp]
\begin{center}
     \begin{subfigure}[b]{0.375\textwidth}
        \centering
        \includegraphics[width=1\textwidth]{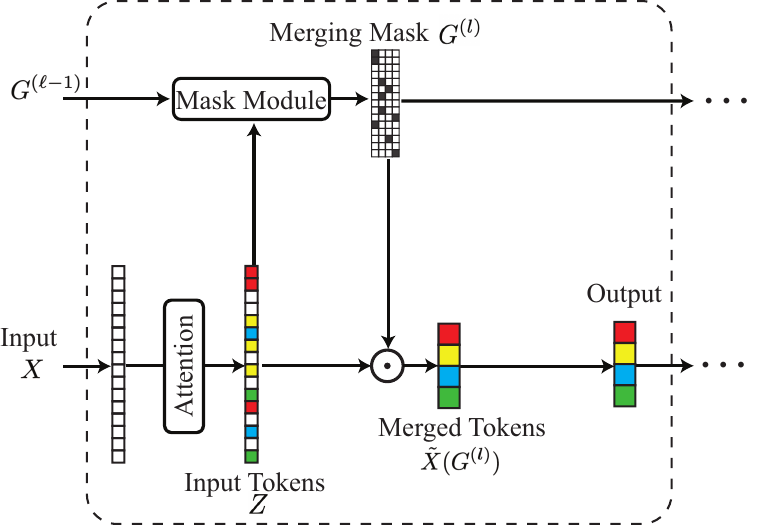}
        \caption{LTM block for regular transformers, such as ViT and Swin.}
        \label{subfig:regular}
    \end{subfigure}
    \hspace{2mm}
    \begin{subfigure}[b]{0.5\textwidth}
        \centering
        \includegraphics[width=1\textwidth]{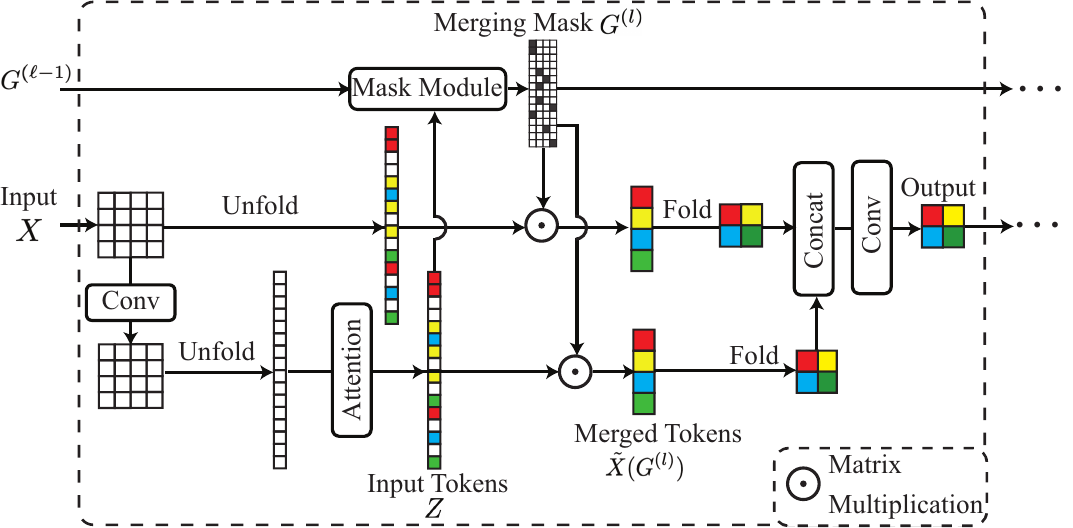}
        \caption{LTM block for efficient transformers, such as MobileViT and EfficientViT.}
        \label{subfig:efficient}
    \end{subfigure}
\end{center}
\vspace{-3mm}
\caption{Overall framework of Learnable Token Merging (LTM)-Transformer block for regular transformer blocks such as ViT and Swin (a), and
  efficient transformer blocks such as MobileViT and EfficientViT (b).}
  \vspace{-3mm}

\label{fig:ltm}
\end{figure*}
\section{Formulation}
\label{sec:formulation}

In this section, we first illustrate  how to perform
token merging using a token merging mask.
We then describe how to generate the token merging mask from
a learnable mask module in a LTM-Transformer block, as well as
the training algorithm of a neural network with
LTM-Transformer blocks. We derive a novel and separable variational upper bound for  the IB loss, and the token merging masks are
generated to reduce such variational upper bound for
the IB loss.

\subsection{Token Merging by Learnable Token Merging Masks}
\label{sec:channel-selection-attention-weights}
Given the input feature tokens $X\in \RR^{N\times D}$ where $N$ is the number of  tokens and $D$ is the token dimension, the
LTM block first applies the self-attention module on the input feature tokens by $Z=\text{ATTN}(X) \in \RR^{N \times D}$, where $\text{ATTN}(\cdot)$ is the regular QKV self-attention operation~\cite{ViT}. As illustrated in
Figure~\ref{fig:ltm}, every LTM block
has a learnable mask module that generates the token merging mask
$G^{(\ell)}$ where $\ell$ is the index of the current layer or block.
The LTM block merges the $N$ tokens of $Z$
into $P$ tokens with $P < N$ by multiplying $Z$ with
the token merging mask $G^{(\ell)}\in\RR^{N\times P}$.
We set $P=\lceil r\times N \rceil$,
where $r\in(0,1)$ is termed the compression ratio for LTM, and a smaller $r$ renders less merged tokens after token merging. The token merging mask $G^{(\ell)}$ of the $\ell$-th transformer block is generated by the token merging mask $G^{(\ell-1)}$ of the previous layer and the feature tokens $Z$, which is motivated by reducing the IB loss and
detailed in Section~\ref{sec:LTM-token-mask-optimization}.
The token merging mask $G^{(1)}$ for the first transformer block is generated by applying an existing learnable token merging method, LTMP~\cite{bonnaerens2023learned}, which generates a binarized token merging mask $M\in [0,1]^{N\times P}$ using Gumbel-Softmax with $N\times P$ learnable parameters.
After obtaining the merging mask $G^{(\ell)}$, the features tokens of $Z$ are merged into $P$ tokens by $\tilde{X}(G^{(\ell)}) = \pth{Z^{\top}G^{(\ell)}}^{\top}\in\mathbb{R}^{P\times D}$, which is then passed to the following MLP layers in the transformer block.


In addition to merging tokens in regular transformer blocks such as ViT~\cite{ViT} and Swin~\cite{liu2021swin}, the LTM-Transformer block
can also be applied to efficient transformer blocks
widely applied in efficient visual transformer architectures
such as MobileViT~\cite{mobilevit} and
EfficientViT~\cite{cai2022efficientvit}. Regular transformer blocks obtain the output by sequentially applying the attention operation and the MLP on the input feature tokens.
However, efficient transformer blocks usually contain residual connections following the design of residual connections in Convolutional Neural Networks (CNNs). That is, these blocks maintain the same shapes for the input $X$ and the self-attention output $Z$ and concatenate them to produce the output features of the current transformer block. As a result, we cannot only merge the tokens of $Z$. Instead, our LTM-Transformer block merges the tokens of both $X$ and $Z$ so that the number of merged tokens for $X$ and $Z$ is the same.
To this end, we apply the same token merging mask $G^{(\ell)}$ to merge both $X$ and $Z$. As a result, the compressed $X$ and $Z$ are of the same shape after the token merging process, and they can still be concatenated, which is illustrated in Figure~\ref{subfig:efficient}. In addition, transformer blocks in the efficient visual transformers are usually accompanied with convolution operations so that they need to maintain the feature tokens in a three-dimensional format $X\in\mathbb{R}^{H\times W\times D}$ as illustrated in Figure~\ref{subfig:efficient}. To apply our token merging method on efficient transformer blocks, we set the number of merged tokens after token merging as $P=H'\times W'$, where $r$ is the compression ratio, and
$H'=\lceil H\times \sqrt{r}\rceil,W'=\lceil W\times\sqrt{r}\rceil$. Therefore, the merged tokens can still be reshaped into three-dimensional features for later convolution operations.
The analysis about the inference computation cost, or the FLOPs, of the LTM transformer block for token merging in both regular transformers and efficient transformers as illustrated in Figure~\ref{fig:ltm} is detailed below.

\noindent \textbf{Computation Cost Analysis of LTM-Transformer for Token Merging.} We hereby analyze the additional inference computation cost, or the FLOPs, of the LTM transformer block for token merging in both regular transformers and efficient transformers as illustrated in Figure~\ref{fig:ltm}. Let $D$ be the dimension of input tokens and $N$ be the number of tokens. The FLOPs of the token merging in an LTM transformer block in regular visual transformers is $6CDP+3C+ND^2+NDP$, where $6CDP+3C+ND^2$ is the FLOPs for calculating the merging mask and $NDP$ is the cost for applying the merging mask on the input tokens. In the LTM transformer block of efficient visual transformers, the additional FLOPs of the token merging is $6CDP+3C+ND^2+2NDP$, since the merging mask will be applied to both the input tokens and the merged tokens.



\subsection{Generating Token Merging Mask by Reducing the Variational Upper Bound for the IB Loss}
\label{sec:LTM-token-mask-optimization}
We describe how to generate the token merging mask in a
LTM-Transformer block in this subsection, and the generation of the token
merging mask is inspired by reduction of the IB loss. We first introduce the setup where the IB loss can be specified.

Given the training data $\set{X_i,y_i}_{i=1}^n$ where $X_i$
is the $i$-th input training feature and $y_i$ is the corresponding class label. Let $Z_i$ be the self-attention output tokens of
the $X_i$, and $\tilde X_i(G) = \pth{Z_i G}^{\top}$ is the merged tokens
with $G$ being the token merging mask.
We first specify how to compute the IB loss,
$\textup{IB}(G) = I(\tilde X(G),X)
-I(\tilde X(G),Y)$ which depends on  $G$ and other network
parameters, $X$ is a random variable representing the input feature which takes values in $\set{X_i}_{i=1}^n$,
$\tilde X(G)$ is a random variable representing
the merged tokens
 which takes values in $\set{\tilde X_i(G)}_{i=1}^n$.
 $Y$ is a random variable representing the class label
 which takes values in $\set{y_i}_{i=1}^n$. We first compute the class centroids $\set{\tilde \cC_a}_{a=1}^C$ and $\set{\cC_b}_{b=1}^C$ for the
merged tokens and the input features by averaging the merged tokens and the input features in each class, where $C$ is the number of classes. We also abbreviate $\tilde X(G)$ as $\tilde X$ for simplicity of the notations.  Then we define the probability that $\tilde X$ belongs to class $\tilde \cC_a$ as $\Prob{\tilde X \in a} = \frac 1n \sum\limits_{i=1}^n  \phi(\tilde X_i,a)$ with
$\phi(\tilde X_i,a) = \frac{\exp\left(-\ltwonorm{\tilde X_i -
\tilde \cC_a}^2\right)}{\sum_{a=1}^{A}\exp\left(-\ltwonorm{\tilde X_i - \tilde \cC_a}^2\right)}$. Similarly, we define the probability that $X_i$ belongs to class $\cC_b$
as $\Prob{X \in b}
= \frac 1n \sum\limits_{i=1}^n  \phi(X_i,b)$. Moreover, we have the joint probabilities $\Prob{\tilde X \in a, X \in b}
= \frac 1n \sum\limits_{i=1}^n  \phi(\tilde X_i,a) \phi(X_i,b)$ and
$\Prob{\tilde X \in a, Y = y}
= \frac 1n \sum\limits_{i=1}^n \phi(\tilde X_i,a) \indict{y_i = y}$ where $\indict{}$ is an indicator function. As a result, we can compute the mutual information $I(\tilde X(G), X)$
and $I(\tilde X(G), Y)$ by
\noindent\resizebox{1\columnwidth}{!}{
    \begin{minipage}{1\columnwidth}
\bals
I(\tilde X(G), X) &= \sum\limits_{a=1}^C \sum\limits_{b=1}^C
\Prob{\tilde X(G) \in a, X \in b} \log{\frac{\Prob{\tilde X(G) \in a, X \in b}}
{\Prob{\tilde X(G) \in a}\Prob{X \in b}}}, \\
I(\tilde X(G), Y) &= \sum\limits_{a=1}^C \sum\limits_{y=1}^C
\Prob{\tilde X(G) \in a, Y = y} \log{\frac{\Prob{\tilde X (G) \in a, Y = y}}
{\Prob{\tilde X (G)\in a}\Prob{Y = y}}},
\eals
        \vspace{1mm}
    \end{minipage}
}
\noindent and then compute the IB loss $\textup{IB}(G)$.
As explained in our motivation, we aim to perform token merging
while can reduce the IB loss. However, directly optimizing the IB loss
in the standard SGD training is difficult as the IB loss
is not separable.
Given a variational distribution $Q(\tilde X \in a| Y=y)$ for $y, a
\in [C]$ computed by Eq. (5) in Section 1.3 of the supplementary, the following theorem gives a variational upper bound,
$\textup{IBB}(G)$, for the IB loss $\textup{IB}(G)$.
$\textup{IBB}(G)$ is separable and thus compatible with SGD
training with minibatches.
\begin{theorem}\label{theorem:IB-upper-bound}
\bal\label{eq:IB-upper-bound}
\textup{IB}(G) \le \textup{IBB}(G) - C_0,
\eal
where $C_0$ is a constant only depending on the input training features
$\set{X_i}_{i=1}^n$, and
\vspace{-.1in}
\bsals
&\textup{IBB}(G) \defeq \frac 1 n \sum\limits_{i=1}^n
\sum\limits_{a=1}^C \sum\limits_{b=1}^C
\phi(\tilde X_i(G),a) \phi(X_i,b)
\log {\phi(X_i,b)}
\nonumber \\
&\phantom{=}-\frac 1n \sum\limits_{i=1}^n\sum\limits_{a=1}^C \sum\limits_{y=1}^C
  \phi(\tilde X_i(G),a) \indict{y_i = y}
  \log{Q(\tilde X \in a| Y=y)}.
\esals
\end{theorem}

\begin{proposition}
\label{proposition:Gmask-gradient-descent}
Suppose $\tilde X_i(G) = \pth{Z_i^{\top} G}^{\top} \in \RR^{P \times D}$ with $Z_i \in \RR^{N \times D}$ being the self-attention output tokens for the $i$-th training feature and $G \in \RR^{N \times P}$ is the token merging mask where $N$ is the number of tokens, $D$ is the token dimension, $P$ is the number of merged tokens after token merging, and $\tilde X_i(G)$ denotes the merged tokens. At step $\ell$ of gradient descent on $\textup{IBB}(G)$, we have
\noindent\resizebox{1\columnwidth}{!}{
    \begin{minipage}{1\columnwidth}
\bal
\label{eq:Gmask-gradient-descent}
G^{(\ell)} &= G^{(\ell-1)} - \eta
\nabla_{G} \textup{IBB}(G^{(\ell-1)}) \nonumber \\
&= G^{(\ell-1)} -   \frac{2\eta}{n} \sum\limits_{i=1}^n
\sum\limits_{a=1}^C Z_i \frac{S_{i{a}}^{(l-1)}}{(\gamma_i^{(l-1)})^2} \left(\gamma_i^{(l-1)}\cC_{a} -\zeta_i^{(\ell-1)} \right)  \psi_{i,a}, ~\ell \ge 2,
\eal
        \vspace{1mm}
    \end{minipage}
}
where $S^{(\ell)}_{ia} \defeq
\exp\pth{-\ltwonorm{\tilde X_i(G^{(\ell)})- \tilde \cC_a}^2}$ for $i \in [n]$ and $a \in [C]$,
$\gamma^{(\ell)}_i \defeq \sum\limits_{a=1}^C S^{(\ell)}_{ia}$,
$\zeta^{(\ell)}_i \defeq \sum\limits_{a=1}^C S^{(\ell)}_{ia} \tilde \cC_a$ for $i \in [n]$,
$\psi_{i,a} \defeq \sum\limits_{b=1}^C \phi(X_i,b)
\log {\phi(X_i,b)} - \sum\limits_{y=1}^C \indict{y_i = y} \log{Q(\tilde X \in a| Y=y)}$.
\end{proposition}
The proofs of Theorem~\ref{theorem:IB-upper-bound}
and Proposition~\ref{proposition:Gmask-gradient-descent}
 are deferred to
Section 1 of the supplementary.
Inspired by Proposition~\ref{proposition:Gmask-gradient-descent},
we can understand a transformer with token merging and multiple transformer
blocks as an iterative process which reduces $\textup{IBU}(G)$
by gradient descent, where the $\ell$-th transformer block
performs one-step gradient descent on $\textup{IBU}(G)$ according to
(\ref{eq:Gmask-gradient-descent}). The mask module of at
the $\ell$-th LTM block generates the token merging
mask $G^{(\ell)}$ from $G^{(\ell-1)}$,
 the token merging mask
of the previous block, through
(\ref{eq:Gmask-gradient-descent}). To improve the flexibility
of the token merging mask, an MLP is applied on $Z_i$.
Moreover, as $\textup{IBU}$ and $\nabla_{G} \textup{IBU}$ are separable, (\ref{eq:Gmask-gradient-descent}) can be performed on a
minibatch $\cB_j \subseteq \set{1,\ldots,n}$, which is compatible with minibatch-based training with SGD for a transformer network with LTM blocks. In practice, the mask module of the $\ell$-th LTM block generates $G^{(\ell)}$
by
\bal
\tilde G^{(\ell)} =& G^{(\ell-1)} \nonumber \\
&-   \frac{2\eta}{n} \sum\limits_{i \in \cB_j}
\sum\limits_{a=1}^C Z_i \frac{S_{i{a}}^{(l-1)}}{(\gamma_i^{(l-1)})^2} \left(\gamma_i^{(l-1)}\cC_{a} -\zeta_i^{(l-1)} \right)  \psi_{i,a},
\label{eq:Gmask-GD-MLP-SGD}
\\
G^{(\ell)}  =& \tilde G^{(\ell)} \circ M^{(\ell)}
\label{eq:Gmask}
\eal
where $M^{(\ell)}\in [0,1]^{N\times P}$ is a binarized token merging mask generated by LTMP~\cite{bonnaerens2023learned} for the $\ell$-th LTM block by applying the Gumbel-Softmax operation on $N\times P$ learnable parameters. The mask $ M^{(\ell)}$ in our LTM serves as a learnable token merging mask module.
Since the update formulation in Eq.(\ref{eq:Gmask-GD-MLP-SGD}) does not incorporate any trainable parameters, the number of the trainable parameters of an LTM block is the same as the number of the trainable parameters in a transformer block with LTMP, which is $N\times P$.
\begin{algorithm}[!htb]
\caption{Training Algorithm of LTM-Transformers}\label{Algorithm_training}
{
\small
\begin{algorithmic}[1]
\STATE Initialize the weights of the network
by $\cW = \cW(0)$ through random initialization,
set $t_{\text{train}}$ which is the number of training epochs.

\FOR{$t\leftarrow 1$ to $t_{\text{train}}$}
\IF{$t<t_{\text{warm}}$}
\STATE Perform gradient descent by a standard step of SGD without applying token merging in LTM transformer blocks.
\ELSE
\STATE Update $\tau(\tilde X_i,a)$ for all the classes $a \in [C]$ and $i \in [n]$.
\FOR{$j \leftarrow 1$ to $J$}
\STATE \textbf{Forward step}: generate $\set{G^{(\ell)} }$ for all the LTM blocks by Eq. (\ref{eq:Gmask-GD-MLP-SGD}) using the minibatch
$\cB_j$, the updated
$\set{\tau(\tilde X_i,a)}_{i \in \cB_j,a \in [C]}$,
$\set{Q^{(t-1)}(\tilde X \in a| Y=y)}_{a \in [C], y \in [C]}$, and $\set{\tilde \cC^{(t-1)}_a}_{a=1}^C$, as well as the output of the network
\STATE \textbf{Backward step}: update the MLP layers of the mask modules in all the LTM blocks  as well as all the other weights in the neural network by a standard step of SGD on the cross-entropy loss
\ENDFOR
\STATE Compute $Q^{(t)}(\tilde X \in a| Y=y)$ by Eq. (5) in Section 1.3 of the supplementary, and update the class centroids $\set{\tilde \cC^{(t)}_a}_{a=1}^C$.
\ENDIF
\ENDFOR
\STATE \textbf{return} The trained weights $\cW$ of the network
\end{algorithmic}
}
\end{algorithm}

The training loss of our LTM models are the regular cross-entropy loss. The derived variational upper
bound for the IB loss, $\textup{IBB}$, is used to design the token mask module, where the token merging mask $G^{(\ell)}$ is generated by Eq. (\ref{eq:Gmask-GD-MLP-SGD}) and  Eq. (\ref{eq:Gmask}).  Algorithm~\ref{Algorithm_training} describes the training process of a neural network with LTM-Transformer blocks using the standard cross-entropy loss for a classification problem. It is remarked that all the MLP layers of the mask modules in all the LTM-Transformer blocks, along with other network parameters, are updated with standard SGD. In order to generate the token merging masks for all the LTM-Transformer blocks before a new epoch starts, we update the variational distribution $Q^{(t)}$ at the end of the previous epoch.

\begin{table*}[!hbt]
    \centering
    \resizebox{\textwidth}{!}{
        \begin{tabular}{lccccccccccc}
            \hline
            \multirow{2}{*}{Methods} & \multirow{2}{*}{\# Params.} & \multirow{2}{*}{FLOPs} &Training Time& Inference Time & \multicolumn{6}{c}{Top-1 Accuracy (\%)} \\ \cline{6-11}
            & & & (minutes/epoch)& (ms/batch) & 0& 1 & 5 & 10 & 25 & 50\\
            \hline
            MobileViT-XS~\cite{mobilevit} & 2.3 M & 0.70  G & -  & 11.3 & 74.80 &- & - & - & - & - \\
            ToMe-MobileViT-XS~\cite{ToMe} & 2.3 M & 0.54  G & -  & 10.4 & 72.73 &- & - & - & - & - \\
            ToFu-MobileViT-XS~\cite{kim2024token} & 2.3 M & 0.54  G & -  & 10.7 & 73.32 &- & - & - & - & - \\
            LTMP-MobileViT-XS\cite{bonnaerens2023learned} & 2.3 M & 0.56 G &  16.0 & 10.9 &- & 73.91 & 73.79 & 73.98 & 74.05 & 74.18 \\
            \textbf{LTM-MobileViT-XS (Fine-tuned)} & 2.3 M & 0.52 G &  16.5 & 10.7 &- & \textbf{74.25} & \textbf{74.31} & \textbf{74.54} & \textbf{74.70} & \textbf{74.95} \\
            \hline
            MobileViT-S~\cite{mobilevit} & 5.6 M & 1.40  G & -  & 15.1 & 78.40 &  -& - & - & - & - \\
            ToMe-MobileViT-S~\cite{ToMe} & 5.6 M & 1.22  G & -  & 14.2 & 76.72 &  -& - & - & - & - \\
            ToFu-MobileViT-S~\cite{kim2024token} & 5.6 M & 1.22 G &  - & 14.4 & 77.24 &  -& - & - & - & - \\
            LTMP-MobileViT-S\cite{bonnaerens2023learned} & 5.6 M & 1.26 G &  18.1 & 14.5 & -& 77.53 & 77.69 & 77.82 & 78.03 & 78.14 \\
            \textbf{LTM-MobileViT-S (Fine-tuned)} & 5.6 M & 1.17 G &  19.0 & 14.1 &- &  \textbf{77.72} & \textbf{78.15} & \textbf{78.34} & \textbf{78.85} & \textbf{79.05} \\
            \hline
            EfficientViT-B1 [r224]~\cite{cai2022efficientvit} & 9.1 M & 0.52  G & -  & 10.0& 79.40 & - & - & - & - & - \\
            ToMe-EfficientViT-B1 [r224]~\cite{ToMe} & 9.1 M & 0.47  G &  - & 9.6& 78.81 & - & - & - & - & - \\
            ToFuEfficientViT-B1 [r224]~\cite{kim2024token} & 9.1 M & 0.47  G & -  & 9.8& 79.04 & - & - & - & - & - \\
            LTMP-EfficientViT-B1 [r224]\cite{bonnaerens2023learned} & 9.1 M & 0.50 G &  13.8 & 9.8 & - &79.21 & 79.31  & 79.32 & 79.36  &  79.40\\
            \textbf{LTM-EfficientViT-B1 [r224] (Fine-tuned)} & 9.1 M & 0.44 G & 14.6  & 9.6 & -&\textbf{79.39} & \textbf{79.62} & \textbf{79.85} & \textbf{80.07} & \textbf{80.22} \\
            \hline
            Swin-T~\cite{liu2021swin} & 29.0 M & 4.50 G & -  & 20.8& 81.30 & - & - & - & - & - \\
            ToMe-Swin-T~\cite{ToMe} & 29.0 M & 3.91 G &  - & 17.5& 79.28 & - & - & - & - & - \\
            ToFuSwin-T~\cite{kim2024token} & 29.0 M & 3.91 G & -  & 17.8& 79.65 & - & - & - & - & - \\
            LTMP-Swin-T \cite{bonnaerens2023learned} & 29.0 M & 3.95 G & 21.0  & 17.9 & - & 79.78 & 79.96 & 80.09 & 80.24 & 80.30 \\
            \textbf{LTM-Swin-T (Fine-tuned)} & 29.0 M & 3.82 G & 21.9  & 17.0 & -&\textbf{80.06} & \textbf{80.46} & \textbf{80.79} & \textbf{81.20} & \textbf{81.38} \\
            \hline
            Swin-B~\cite{liu2021swin} & 88.0 M & 15.4 G &  - & 33.9 & 83.50& - & - & - & - & - \\
            ToMe-Swin-B~\cite{ToMe} & 88.0 M & 13.0 G &  - & 29.9 & 81.87 & - & - & - & - & - \\
            ToFu-Swin-B~\cite{kim2024token} & 88.0 M & 13.0 G &  - & 30.1 & 82.04& - & - & - & - & - \\
            LTMP-Swin-B \cite{bonnaerens2023learned} & 88.0 M & 13.2 G &  27.2 & 30.4 & -& 82.24 & 82.39 & 82.45 & 82.51 & 82.55 \\
            \textbf{LTM-Swin-B (Fine-tuned)} & 88.0 M & 12.0 G & 28.9  & 29.6 & - & \textbf{82.50} & \textbf{82.72} & \textbf{82.88} & \textbf{83.43} & \textbf{83.64} \\
            \hline
            ViT-S~\cite{ViT} & 22.1 M & 4.30 G &  - & 22.5& 81.20 & - & - & - & - & - \\
            ToMe-ViT-S~\cite{ToMe} & 22.1 M& 3.82 G & -  & 18.4& 80.04 & - & - & - & - & - \\
            ToFu-ViT-S~\cite{kim2024token} & 22.1 M & 3.82 G &  - & 18.7& 80.19 & - & - & - & - & - \\
            LTMP-ViT-S \cite{bonnaerens2023learned} & 22.1 M & 3.89 G & 21.4  & 19.0 &- & 80.32 & 80.40 & 80.35 & 80.41 & 80.50 \\
            \textbf{LTM-ViT-S (Fine-tuned)} & 22.1 M & 3.70 G & 22.7  & 18.2 &- & \textbf{80.47} & \textbf{80.69} & \textbf{80.94} & \textbf{81.27} & \textbf{81.55} \\
            \hline
            ViT-B~\cite{ViT} & 86.5 M & 17.58 G &  - & 37.2 &83.74 & -& - & - & - & - \\
            ToMe-ViT-B~\cite{ToMe} & 86.5 M & 13.12 G &  - & 31.0 &82.86 & -& - & - & - & - \\
            ToFu-ViT-B~\cite{kim2024token} & 86.5 M & 13.12 G &  - & 31.5 &83.22 & -& - & - & - & - \\
            LTMP-ViT-B \cite{bonnaerens2023learned} & 86.5 M & 13.46 G & 27.2  & 32.7 & -& 83.29 & 83.40 & 83.44 & 83.50 & 83.55 \\
            \textbf{LTM-ViT-B (Fine-tuned)} & 86.5 M & 12.85 G &  28.4 & 30.7 &- & \textbf{83.35} & \textbf{83.57} & \textbf{83.76} & \textbf{83.91} & \textbf{83.96} \\
            \hline
        \end{tabular}
    }
    \vspace{-2mm}
    \caption{Performance comparison between LTM with competing token merging baselines, ToMe~\cite{ToMe}, ToFu~\cite{kim2024token}, and LTMP~\cite{bonnaerens2023learned} in fine-tuning setup on ImageNet. Among the compared methods, ToMe and ToFu do not require training. Both LTM models and LTMP models are fine-tuned for 1, 5, 10, 25, and 50 epochs for fair comparisons.}
    \label{tab:finetune}
\end{table*}
\vspace{-4mm}
\section{Experimental Results}
\label{sec:experiments}

In this section, LTM-Transformers are assessed for the image classification task on the ImageNet-1k dataset. The results in Section~\ref{sec:classification} indicate that LTM outperforms existing state-of-the-art networks while maintaining a more compact architecture either fine-tuned from pre-trained backbones or trained from scratch. In addition, LTM is compared with existing methods on token merging and shows better performance with lower computation costs. Furthermore, in Sections~\ref{sec:object_detection} and~\ref{sec:segmentation}, we demonstrate that the use of LTM-MobileViT and LTM-EfficientViT as feature extraction backbones leads to superior mAP and reduced FLOPs compared to the baseline models for the tasks of object detection and semantic segmentation.
Comprehensive ablation studies are performed in Section~\ref{sec:ablation-study}.
In Section~\ref{sec:study_compression}, we study the impact of compression ratio on LTM.
In Section~\ref{sec:ablation-study_IB}, we perform ablation studies on the effects of LTM-Transformer in reducing the IB loss
and the IB bound. In Section~\ref{sec:IB_different_layers}, we study IB loss and IB bound at different layers of an LTM-Transformer. In Section~\ref{sec:train-test-loss-LTM}, we illustrate the training loss and the test loss in the training process of LTM-Transformers. In Section~\ref{sec:visual_results}, we visualize merged tokens and their merging weights by LTM for selected images from ImageNet. In Section~\ref{sec:layer_compression_ratio}, we study the impact of compression ratio at different layers of the LTM-Transformer. Additional experiment results are presented in Section 2 of the supplementary of our paper. In Section 2.1 of the supplementary, we evaluate the transfer learning capability of the LTM-Transformer. In Section 2.2 of the supplementary, we evaluate the effectiveness of fine-tuning the LTM-Transformer with LoRA. In Section 2.3 of the supplementary, we evaluate the throughput of the LTM-Transformer with variant batch sizes.

\vspace{-.1in}
\subsection{Image Classification}
\label{sec:classification}

\textbf{Implementation details.}
In this section, we evaluate LTM models for ImageNet~\cite{russakovsky2015imagenet} classification. We employ MobileViT-S~\cite{mobilevit}, MobileViT-XS~\cite{mobilevit}, EfficientViT-B1~\cite{cai2022efficientvit}, ViT-S~\cite{ViT},  ViT-B~\cite{ViT}, Swin-T~\cite{liu2021swin}, and Swin-B~\cite{liu2021swin} as backbone architectures. We substitute the conventional transformer blocks in these backbones with LTM blocks.  All the experiments are conducted on four NVIDIA A100 GPUs with a total batch size of $1024$ images. Following prior works~\cite{liu2021swin}, our training incorporates popular data augmentation methods such as RandAugment, Mixup, Cutmix, and random erasing.  We set $\eta$ in Eq. (\ref{eq:Gmask-gradient-descent}) to $1$ in all the experiments. In addition, we apply a softmax operation on the token merging mask at each layer to ensure the aggregation weights for each merged token sum to $1$. In all our experiments, we set the value of compression ratio $r=0.7$ for all our LTM models. We set the compression ratio of the compared methods to $0.75$ to see how LTM-Transformers compare with competing token merging methods with an even more significant reduction in the number of tokens. A study on the impact of the compression ratio $r$ on the performance of the LTM model is performed in Table~\ref{tab:model_comparison} in Section~\ref{sec:study_compression}.

We conduct the experiments of LTM for the token merging of vision transformers under two different training setups, which are the fine-tuning setup and the training-from-scratch setup. The experiments in the fine-tuning setup are conducted following the state-of-the-art token merging method, LTMP~\cite{bonnaerens2023learned}. The training-from-scratch setup is designed to explore the potential of training LTM-Transformers from the beginning while reducing the IB loss with token merging, and the training with different backbones follows the same training settings as the original training process of the corresponding backbones~\cite{mobilevit, cai2022efficientvit, ViT, liu2021swin}.


\begin{table}[!htbp]
    \centering
    \resizebox{1\columnwidth}{!}{
        \begin{tabular}{lrrr}
            \hline
            Model  & \# Params & FLOPs & Top-1 \\
            \hline

             MobileViT-XS  & 2.3 M & 0.7 G & 74.8\\
             ToMe-MobileViT-XS ~\cite{ToMe}  &2.3 M & 0.54 G&72.7\\
             ToFu-MobileViT-XS ~\cite{kim2024token}  &2.3 M & 0.54 G &73.3 \\
             LTMP-MobileViT-XS ~\cite{bonnaerens2023learned}  &2.3 M & 0.56 G  &73.9 \\
             {\bf LTM-MobileViT-XS (Ours) }  & 2.3 M & 0.52 G & \textbf{75.8} \\
            \hline
             MobileViT-S  & 5.6 M & 1.4 G& 78.4 \\
             ToMe-MobileViT-S ~\cite{ToMe}  &5.6 M & 1.22 G&76.7\\
             ToFu-MobileViT-S ~\cite{kim2024token}  &5.6 M & 1.22 G&77.2\\
             LTMP-MobileViT-S ~\cite{bonnaerens2023learned}  &5.6 M & 1.26 G&77.5\\
            {\bf LTM-MobileViT-S (Ours) }& 5.6 M & 1.17 G& \textbf{79.7} \\
            \hline
            EfficientViT-B1 [r224]~\cite{cai2022efficientvit}  & 9.1 M & 0.52 G & 79.4 \\
            S$^2$ViTE-EfficientViT-B1 [r224]~\cite{chen2021chasing}        & 8.2 M & 0.47 G  & 79.0\\
            SPViT-EfficientViT-B1 [r224]~\cite{kong2022spvit}          & 9.2 M     & 0.49 G   & 79.3\\
            SAViT-EfficientViT-B1 [r224]~\cite{zheng2022savit}         & 8.4 M     & 0.47 G   & 79.2\\
            ToMe-EfficientViT-B1 [r224]~\cite{ToMe}  & 9.1 M & 0.47 G &78.8\\
            ToFu-EfficientViT-B1 [r224] ~\cite{kim2024token}  & 9.1 M & 0.47 G &79.0\\
            LTMP-EfficientViT-B1 [r224] ~\cite{bonnaerens2023learned}  & 9.1 M & 0.50 G &79.2\\
            \textbf{LTM-EfficientViT-B1 [r224] (Ours)}  & 9.1 M & 0.44 G & \textbf{80.2} \\
            \hline
            EfficientViT-B1 [r288]~\cite{cai2022efficientvit}  & 9.1 M & 0.86 G & 80.4 \\
            ToMe-EfficientViT-B1 [r288]~\cite{ToMe}  & 9.1 M & 0.73 G &79.7\\
            ToFu-EfficientViT-B1 [r288]~\cite{kim2024token}  & 9.1 M & 0.73 G &79.8\\
            LTMP-EfficientViT-B1 [r288]~\cite{bonnaerens2023learned}  & 9.1 M & 0.76 G &80.0\\
            \textbf{LTM-EfficientViT-B1 [r288] (Ours)}  & 9.1 M & 0.70 G & \textbf{81.0} \\
             \hline
             ViT-S~\cite{dosovitskiy2020image}  & 22.1 M & 4.3 G &81.2\\
             \textbf{LTM-ViT-S (Ours)}~  & 22.1 M & 3.7 G &\textbf{81.8}\\
             ViT-B~\cite{dosovitskiy2020image}  & 86.5 M &17.6 G  & 83.7\\
             \textbf{LTM-ViT-B (Ours)}~  &86.5 M &12.9 G &\textbf{83.9}\\
             Swin-T~\cite{liu2021swin}  & 29.0 M & 4.5 G &81.3 \\
             \textbf{LTM-Swin-T (Ours)}~  & 29.0 M & 3.8 G &\textbf{81.8} \\
             Swin-B~\cite{liu2021swin}  & 88.0 M & 15.4 G &83.5 \\
             \textbf{LTM-Swin-B (Ours)}~  & 88.0 M & 12.0 G  &\textbf{83.8} \\
            \hline
        \end{tabular}
    }
    \vspace{-2mm}
    \caption{Comparisons with baseline methods on ImageNet-1k validation set for the train-from-scratch setup.}
    \label{tab:imagenet_results}
\end{table}
\noindent\textbf{Fine-Tuning Setup.} Our proposed LTM can be straightforwardly applied to token merging with pre-trained models using the fine-tuning setup as in the existing state-of-the-art token merging method, LTMP~\cite{bonnaerens2023learned}. In the fine-tuning setup, LTM models are not trained from scratch, and token merging for a pre-trained visual transformer is performed by simply changing all the transformer blocks of the pre-trained models to LTM-Transformer blocks. Following the settings in LTMP~\cite{bonnaerens2023learned}, the token merging mask modules are added to the original transformer blocks, and all the pre-trained weights are loaded as the initialization for the LTM models. In the fine-tuning process, the pre-trained weights are not updated and only the weights in the token merging mask modules, $\set{M^{(\ell)}}$, are updated.
We fine-tune the LTM models for 1, 5, 10, 25, and 50 epochs, respectively, and compare them with LTMP models fine-tuned for the same number of epochs. Note that LTM models and LTMP models with the same backbones have the same number of parameters. It is observed that LTM models significantly outperform LTMP models fine-tuned for the same number of epochs. For example, the LTM-Swin-B fine-tuned for 50 epochs outperforms the LTMP-Swin-B fine-tuned for 50 epochs by $1.09\%$ in top-1 accuracy. The LTM-Swin-T fine-tuned for 50 epochs outperforms the LTMP-Swin-T fine-tuned for 50 epochs by $1.08\%$ in top-1 accuracy.

In addition, we also compare the LTM with three token merging methods, ToMe~\cite{ToMe}, ToFu~\cite{kim2024token}, and LTMP~\cite{bonnaerens2023learned}, to demonstrate the superiority of our LTM. The results are shown in Table~\ref{tab:finetune}.
The inference time of all the models is also evaluated on the validation set of ImageNet-1k and reported in milliseconds (ms) per batch for an evaluation batch size of 128 on one Nvidia A100 GPU.
We also evaluate the training cost of our LTM models compared with the other fine-tuning based token merging method, LTMP~\cite{bonnaerens2023learned}.  The training time is evaluated on 4 NVIDIA A100 GPUs with an effective batch size of 512 images. We report the average training time per epoch.
The training overhead of LTM-Transformers mainly comes from the computation of $\set{\phi(\tilde X_i,a)}_{i \in \cB_j,a \in [C]}$,
$\set{Q^{(t-1)}(\tilde X \in a| Y=y)}_{a \in [C], y \in [C]}$, and $\set{\tilde \cC^{(t-1)}_a}_{a=1}^C$ as described in Algorithm~\ref{Algorithm_training}. It is observed that LTM only brings a marginal training time overhead compared to the competing token merging method.

\noindent\textbf{Training-from-Scratch Setup.} As shown in Table~\ref{tab:imagenet_results}, LTM models show reduced FLOPs and enhanced accuracy compared to their original visual transformer counterparts. For instance, LTM-MobileViT-S not only reduces the FLOPs of MobileViT-S from $1.4$G to $1.17$G but also improves the top-1 accuracy by $1.3\%$.
Similarly, LTM-Swin-T achieves a $0.5\%$ accuracy increase while lowering the FLOPs from $4.5$G to $3.7$G compared to the original Swin-T.
Moreover, we compare LTM models against models compressed by current state-of-the-art weight pruning methods for efficient visual transformers, including S$^2$ViTE~\cite{chen2021chasing}, SPViT~\cite{kong2022spvit}, and SAViT~\cite{zheng2022savit} on EfficientViT-B1 (r224). To apply S$^2$ViTE, SPViT, and SAViT on EfficientViT-B1 (r224), we first run their pruning process following the standard implementation in their papers~\cite{chen2021chasing, kong2022spvit, zheng2022savit} on the ImageNet training data. After obtaining the pruned networks, we fine-tune the pruned models following the same setting as stated in their papers~\cite{chen2021chasing, kong2022spvit, zheng2022savit}.
It is observed that LTM-EfficientViT-B1 outperforms all pruned models by at least $0.9\%$ in top-1 accuracy with even less FLOPs.
\begin{table}[!htbp]
        \centering
            \resizebox{0.75\columnwidth}{!}{
            \Huge
            \begin{tabular}{lrrr}
                \hline
                Feature backbone & \# Params. & FLOPs& mAP\\
                \hline
                MobileNetv3~\cite{howard2019searching} & 4.9 M &1.4 G& 22.0 \\
                MobileNetv2~\cite{sandler2018mobilenetv2} & 4.3 M & 1.6 G & 22.1 \\
                MobileNetv1~\cite{howard2017mobilenets} & 5.1 M &2.6 G& 22.2 \\
                MixNet~\cite{tan2019mixconv} & 4.5 M &2.2 G& 22.3 \\
                MNASNet~\cite{tan2019mnasnet} & 4.9 M &1.7 G& 23.0 \\
                 YoloV5-N (640$\times$640)~\cite{redmon2017yolo9000}& 1.9 M & 4.5 G&  28.0 \\
                 Vidt~\cite{song2021vidt}& 7.0 M & 6.7 G & 28.7 \\
                MobileViT-XS & 2.7 M &1.7 G& 24.8 \\
                \textbf{LTM-MobileViT-XS(Ours)} & 2.7 M  & 1.5 G & \textbf{25.4} \\
                MobileViT-S & 5.7 M &2.4 G& 27.7 \\
                \textbf{LTM-MobileViT-S(Ours)} & 5.7 M &2.1 G& \textbf{28.4} \\
                EfficientViT & 9.9 M &1.5 G& 28.4 \\
                \textbf{LTM-EfficientViT(Ours)} & 9.9 M &1.4 G& \textbf{28.9} \\
                \hline
            \end{tabular}
        }
        \vspace{-2mm}
\caption{Object detection performance with SSDLite.}
\label{tab:object_detection_results}
\end{table}
        \vspace{-2mm}
\subsection{Object Detection}
\label{sec:object_detection}
\textbf{Implementation details.} We incorporate ImageNet pre-trained models, that are LTM-MobileViT-XS, LTM-MobileViT-S, and LTM-EfficientViT, with the single-shot object detection backbone, SSDLite~\cite{sandler2018mobilenetv2}, to evaluate on the MS-COCO dataset~\cite{lin2014microsoft}, which comprises 117k training images and 5k validation images. We fine-tune all pre-trained LTM-Transformers within the object detection framework at a standard input resolution of $320\times 320$. These models undergo a training period of 200 epochs using the AdamW optimizer, adhering to the training protocols established in~\cite{mobilevit}. Employing a cosine learning rate scheduler, the initial learning rate of $0.0009$ is gradually reduced to $1.6e^{-6}$. For the object localization, we utilize a smooth $\ell^1$ loss, and for classification, cross-entropy losses are applied. The evaluation of performance on the validation set is conducted using the mAP metric with an IoU range from 0.50 to 0.95 in increments of 0.05.
\begin{table}[!htbp]
 \small
 \centering
  \resizebox{1\linewidth}{!}{
\begin{tabular}{lcccccc}
\hline
Methods           & mAP$^{box}$ & AP$^{b}_{50}$  & AP$^{b}_{75}$  & mAP$^{m}$ & AP$^{m}_{50}$  & AP$^{m}_{75}$   \\
\hline
EViT~\cite{liu2023efficientvit}          & 32.8               & 54.4                            & 34.5                            & 31.0                & 51.2                             & 32.2                              \\
EfficientViT-B1~\cite{cai2022efficientvit}  & 33.5               & 55.4                            & 34.8                            & 31.9                & 52.3                             & 32.7                              \\
LTM-EfficientViT-B1 & \textbf{34.3}              & \textbf{56.1}                            & \textbf{35.2}                            & \textbf{32.8}                & \textbf{52.8} &\textbf{33.1}\\
\hline
 \end{tabular}
 }
 \vspace{-2mm}
 \caption{Instance Segmentation Results on COCO.}
 \label{tab:seg_results}
\end{table}

\noindent\textbf{Results.}
We conduct a comparative study of our LTM Transformers against other lightweight feature backbones within the SSDLite object detection framework. The results, as detailed in Table~\ref{tab:object_detection_results}, illustrate significant improvements in object detection performance when the feature backbone is upgraded to include LTM-Transformer blocks. For example, substituting MobileViT-S with LTM-MobileViT-S enhances the mAP by $0.7\%$ while concurrently reducing FLOPs by $0.3$G. In addition, SSDLite equipped with LTM-EfficientViT achieves a substantial performance increase of $6.9\%$ while maintaining the same FLOPs as MobileNetV3.

\subsection{Instance Segmentation}
\label{sec:segmentation}
In this section, we assess the efficacy of LTM when applied to instance segmentation tasks using the COCO dataset~\cite{lin2014microsoft}. We utilize Mask R-CNN~\cite{he2017mask} equipped with a Feature Pyramid Network (FPN) as the segmentation head, built on the LTM-EfficientViT-B1 feature backbone. For comparative analysis, we include EfficientViT-B1~\cite{cai2022efficientvit} and EViT~\cite{liu2023efficientvit} as baseline models. Both our models and the baselines are trained on the training split of the COCO dataset and evaluated on the validation split, adhering to the protocols established by~\cite{chen2019mmdetection}. The training duration is set to 12 epochs, consistent with the 1$\times$ schedule described in~\cite{chen2019mmdetection}. The AdamW optimizer is employed for training following the practices of~\cite{liu2023efficientvit}. We initiate the learning rate at $0.001$, which is then gradually reduced following a cosine learning rate schedule. Performance metrics reported include the mean bounding box Average Precision (mAP$^b$) and mean mask Average Precision (mAP$^m$), along with bounding box Average Precision (AP$^b$) and mask Average Precision (AP$^m$) at IoU thresholds of 0.5 and 0.75. The findings, detailed in Table~\ref{tab:seg_results}, demonstrate that LTM-EfficientViT-B1 consistently enhances segmentation performance across various thresholds. For example, LTM-EfficientViT-B1 outperforms EfficientViT-B1 by $0.8\%$ and $0.9\%$ in mAP$^b$ and mAP$^m$, respectively.
\subsection{Ablation Study}
\label{sec:ablation-study}

\subsubsection{Study on the Impact of Compression Ratio}
\label{sec:study_compression}
We conduct an ablation study on the compression ratio of token merging on ViT-B. We train LTM-ViT-B models with different compression ratios in both the train-from-scratch setup and fine-tuning setup. In the fine-tuning setup, all models are fine-tuned for 50 epochs.
It is observed from the results in Table~\ref{tab:model_comparison} that although a smaller compression ratio can result in a slight accuracy drop, the LTM-ViT-B with a compression ratio of 0.65 can still achieve better performance than the original ViT-B model either fine-tuned from the pre-trained checkpoint or trained from scratch.
\begin{table}[!htbp]
\centering
\resizebox{1\columnwidth}{!}{%
\begin{tabular}{lcccc}
\hline
Methods    & FLOPs (G) & $r$ & Train-from-scratch & Fine-tuning  \\ \hline
ViT-B      & 17.58     & 1.00                  & 83.74              & 83.74       \\
LTM-ViT-B & 16.55     & 0.95                  & 84.43              & 84.32       \\
LTM-ViT-B & 15.25     & 0.90                  & 84.46              & 84.29       \\
LTM-ViT-B & 14.19     & 0.85                  & 84.33              & 84.35       \\
LTM-ViT-B & 14.89     & 0.80                  & 84.15              & 84.20       \\
LTM-ViT-B & 13.49     & 0.75                  & 83.95              & 84.05       \\
LTM-ViT-B & 12.85     & 0.70                  & 83.87              & 83.96       \\
LTM-ViT-B & 11.95     & 0.65                  & 83.75              & 83.81       \\
LTM-ViT-B & 11.03     & 0.60                  & 83.53              & 83.56       \\
LTM-ViT-B & 10.15     & 0.55                  & 83.07              & 83.14       \\
LTM-ViT-B & 9.63      & 0.50                  & 82.87              & 82.95       \\
LTM-ViT-B & 8.77      & 0.45                  & 83.46              & 83.51       \\
LTM-ViT-B & 8.30      & 0.40                  & 82.23              & 82.37       \\ \hline
\end{tabular}
}
\vspace{-2mm}
\caption{Performance comparison between LTM-ViT-B with different compression ratios. The models in the fine-tuning setup are fine-tuned for 50 epochs.}
\label{tab:model_comparison}
\end{table}

\subsubsection{Study on the Effects of LTM in Reducing IB Loss}
\label{sec:ablation-study_IB}
We study the effectiveness of LTM in reducing the IB loss and the variational upper bound of IB loss, which is the IB bound, across three vision transformers, including MobileViT-S, MobileViT-XS, and EfficientViT (r224). We compare the performance of the vision transformers with the baseline token merging method, ToME~\cite{ToMe}, LTMP~\cite{bonnaerens2023learned}, and the corresponding LTM-Transformer models with all the transformer blocks replaced with the LTM blocks.
The ablation study results for the fine-tuning setup are shown in Table~\ref{tab:ablation-IB-loss_finetune}.  The ablation study results for the train-from-scratch setup are shown in Table~\ref{tab:ib_loss_ablation_appendix}. The results indicate that although ToMe and LTMP reduce the IB loss and the IB bound in both the fine-tuning setup and the train-from-scratch setup, thereby adhering to the IB principle, which aims to enhance the correlation of features with class labels while reducing their correlation with the input, LTM can further decrease the IB loss and IB bound. In particular, our LTM models improve the vanilla vision transformers, the ToMe models, and the LTMP models by a large margin in terms of both IB loss and top-1 accuracy for both the fine-tuning setup and the train-from-scratch setup. For instance, the LTMP-ViT-B fine-tuned for 50 epochs reduces the IB loss of ViT-B by $0.00333$. The LTM-ViT-B fine-tuned for 50 epochs further reduces the IB loss of LTMP-ViT-B fine-tuned for 50 epochs by $0.00866$.
\begin{table*}[!htbp]
    \centering
    \begin{minipage}{1\textwidth}
        \centering
        \resizebox{1\columnwidth}{!}{
\begin{tabular}{l|c|cccc|cccc|cccc}
\hline
\multirow{2}{*}{Model} & \multirow{2}{*}{FLOPs} & \multicolumn{4}{c|}{Top-1}                                & \multicolumn{4}{c|}{IB Bound}                                     & \multicolumn{4}{c}{IB Loss}                                          \\ \cline{3-14}
                       &                        & 0     & 1              & 10             & 50             & 0       & 1                & 10               & 50               & 0        & 1                 & 10                & 50                \\ \hline
MobileViT-S            & 1.40 G                 & 78.40 & -              & -              & -              & 0.05782 & -                & -                & -                & -0.00432 & -                 & -                 & -                 \\
ToMe-MobileViT-S       & 1.22 G                 & 76.72 & -              & -              & -              & 0.04931 & -                & -                & -                & -0.00525 & -                 & -                 & -                 \\
LTMP-MobileViT-S       & 1.26 G                 & -     & 77.53          & 77.82          & 78.14          & -       & 0.04902          & 0.04735          & 0.04542          & -        & -0.00765          & -0.00874          & -0.00913          \\
\textbf{LTM-MobileViT-S}       & 1.17 G                 & -     & \textbf{77.72} & \textbf{78.34} & \textbf{79.05} & -       & \textbf{0.03095} & \textbf{0.02967} & \textbf{0.02683} & -        & \textbf{-0.01430} & \textbf{-0.01576} & \textbf{-0.01692} \\ \hline
EfficientViT-B1        & 0.52 G                 & 79.40 & -              & -              & -              & 0.06014 & -                & -                & -                & -0.00451 & -                 & -                 & -                 \\
ToMe-EfficientViT-B1   & 0.47 G                 & 78.81 & -              & -              & -              & 0.04642 & -                & -                & -                & -0.00732 & -                 & -                 & -                 \\
LTMP-EfficientViT-B1   & 0.52 G                 & -     & 79.21          & 79.32          & 79.40          & -       & 0.04537          & 0.04219          & 0.03970          & -        & -0.00802          & -0.00916          & -0.00995          \\
\textbf{LTM-EfficientViT-B1}  & 0.44 G                 & -     & \textbf{79.39} & \textbf{79.62} & \textbf{80.22} & -       & \textbf{0.02874} & \textbf{0.02703} & \textbf{0.02635} & -        & \textbf{-0.01585} & \textbf{-0.01664} & \textbf{-0.01704} \\ \hline
ViT-B                  & 17.58 G                & 83.74 & -              & -              & -              & 0.05539 & -                & -                & -                & -0.00419 & -                 & -                 & -                 \\
ToMe-ViT-B             & 13.12 G                & 82.86 & -              & -              & -              & 0.04583 & -                & -                & -                & -0.00647 & -                 & -                 & -                 \\
LTMP-ViT-B             & 13.46 G                & -     & 83.29          & 83.44          & 83.55          & -       & 0.04392          & 0.04275          & 0.04086          & -        & -0.00665          & -0.00693          & -0.00752          \\
\textbf{LTM-ViT-B}             & 12.85 G                & -     & \textbf{83.35} & \textbf{83.76} & \textbf{83.96} & -       & \textbf{0.03732} & \textbf{0.03506} & \textbf{0.03082} & -        & \textbf{-0.01425} & \textbf{-0.01572} & \textbf{-0.01618} \\ \hline
\end{tabular}
        }
        \vspace{-2mm}
        \caption{Ablation study on the effects of LTM in reducing IB loss in the fine-tuning setup. Both LTMP models and LTM models fine-tuned for 1, 10, 50 epochs are evaluated.}
        \label{tab:ablation-IB-loss_finetune}
    \end{minipage}%
\end{table*}

\begin{table}[!htbp]
 \small
 \centering
 \resizebox{1\columnwidth}{!}{
        \begin{tabular}{lccccc}
            \hline
            Model                    & FLOPs  & Top-1         & IB Bound           & IB Loss  \\
            \hline
            MobileViT-S              & 1.40 G & 78.40         & 0.05782  & -0.00432 \\
            ToMe-MobileViT-S         & 1.22 G & 76.72         & 0.04931  & -0.00525 \\
            LTMP-MobileViT-S         & 1.26 G & 78.14         & 0.04542  & -0.00913 \\
            LTM-MobileViT-S         & 1.17 G & \textbf{79.68} & \textbf{0.02425} & \textbf{-0.01725} \\
            \hline
            EfficientViT-B1          & 0.52 G & 79.40         & 0.06014  & -0.00451 \\
            ToMe-EfficientViT-B1     & 0.47 G & 78.81         & 0.04642  & -0.00732 \\
            LTMP-EfficientViT-B1     & 0.52 G & 79.40         & 0.03970 & -0.00995 \\
            LTM-EfficientViT-B1     & 0.44 G & \textbf{80.20} & \textbf{0.02689} & \textbf{-0.01730} \\
            \hline
            ViT-B                    & 17.58 G & 83.74         & 0.05539  & -0.00419 \\
            ToMe-ViT-B               & 13.12 G & 82.86         & 0.04583  & -0.00647 \\
            LTMP-ViT-B               & 13.46 G & 83.55         & 0.04086  & -0.00752 \\
            LTM-ViT-B               & 12.85 G & \textbf{83.87} & \textbf{0.03094} & \textbf{-0.01636} \\
            \hline
        \end{tabular}
 }
 \vspace{-3mm}
 \caption{Ablation study on the effects of LTM in reducing IB loss in the train-from-scratch setup.  }
 \vspace{-3mm}
 \label{tab:ib_loss_ablation_appendix}
\end{table}

\begin{figure*}[!t]
\begin{center}
     \begin{subfigure}[b]{0.375\textwidth}
        \centering
        \includegraphics[height=0.55\textwidth]{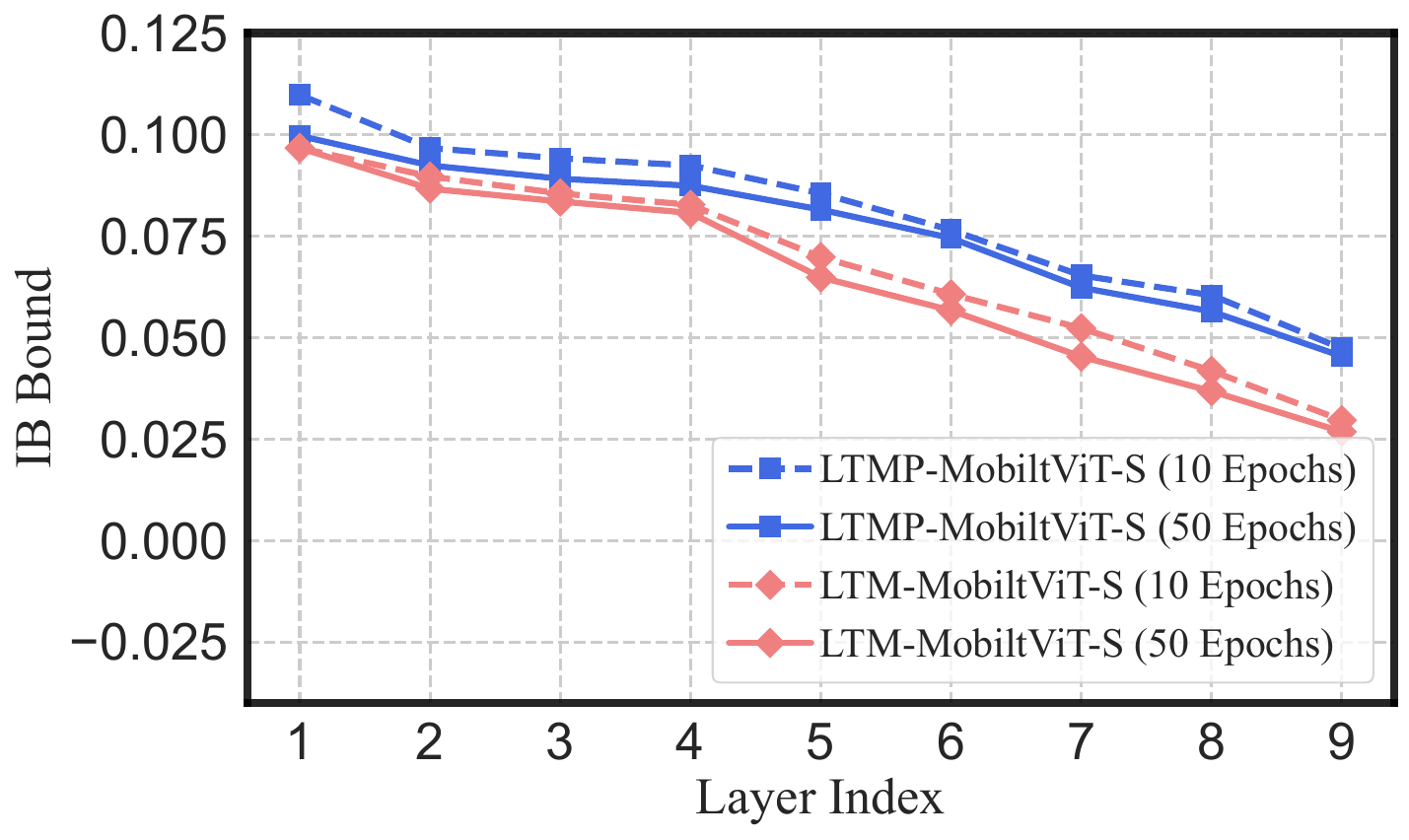}
        \vspace{-2mm}
        \caption{IB bound ($\textup{IBB}(G)$) comparison between MobileViT-S and LTM-MobileViT-S.}
    \end{subfigure}
    \hspace{6mm}
    \begin{subfigure}[b]{0.375\textwidth}
        \centering
        \includegraphics[height=0.55\textwidth]{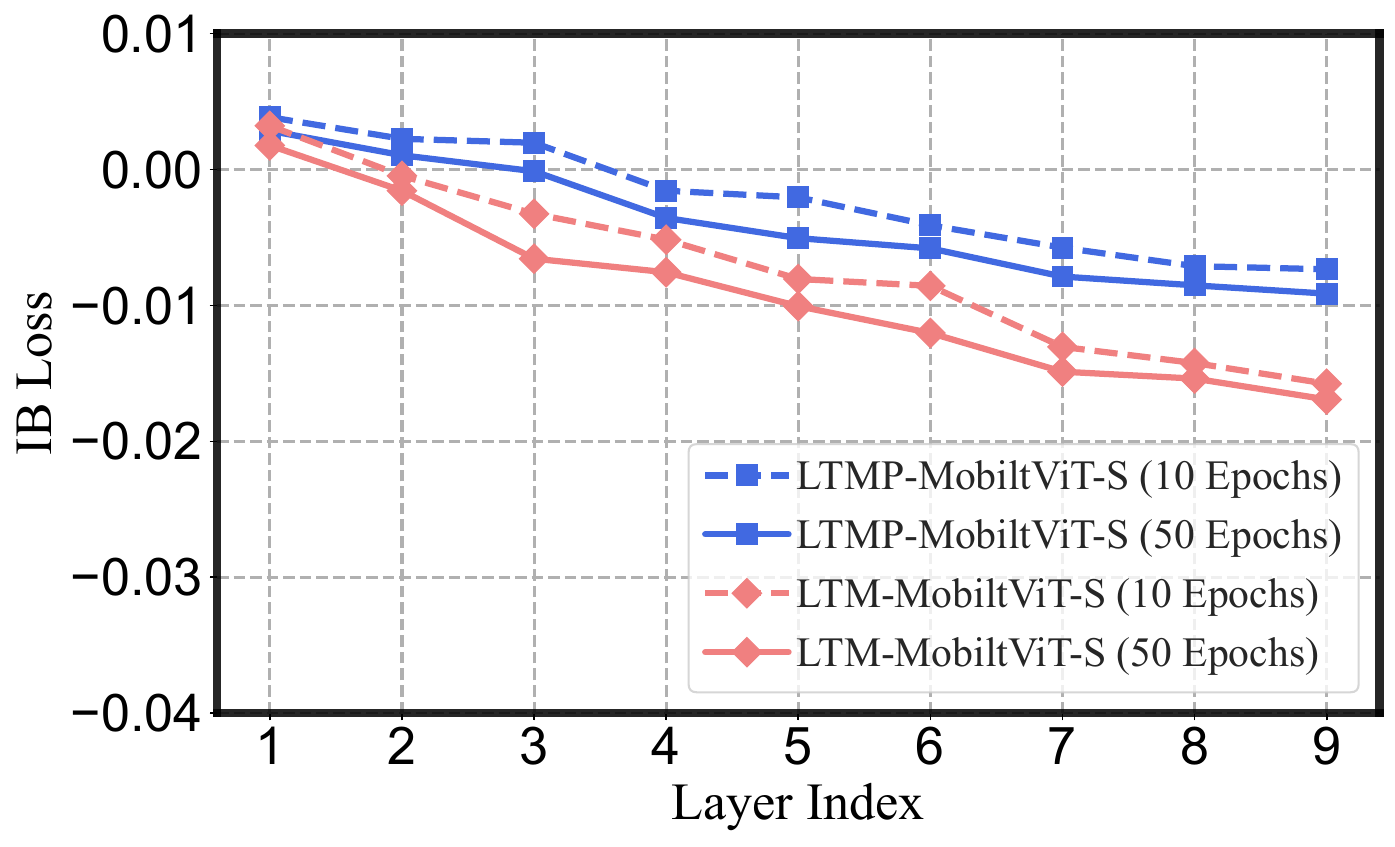}
        \vspace{-2mm}
        \caption{IB loss ($\textup{IB}(G)$) comparison between MobileViT-S and LTM-MobileViT-S.}
    \end{subfigure}
\end{center}
\vspace{-5mm}
\caption{IB bound and IB loss comparison between MobileViT-S and LTM-MobileViT-S at different transformer layers.}
\label{fig:IB_plots}
\end{figure*}

\begin{figure*}[!t]
\begin{center}
     \begin{subfigure}[b]{0.225\textwidth}
        \centering
        \includegraphics[height=0.7\columnwidth]{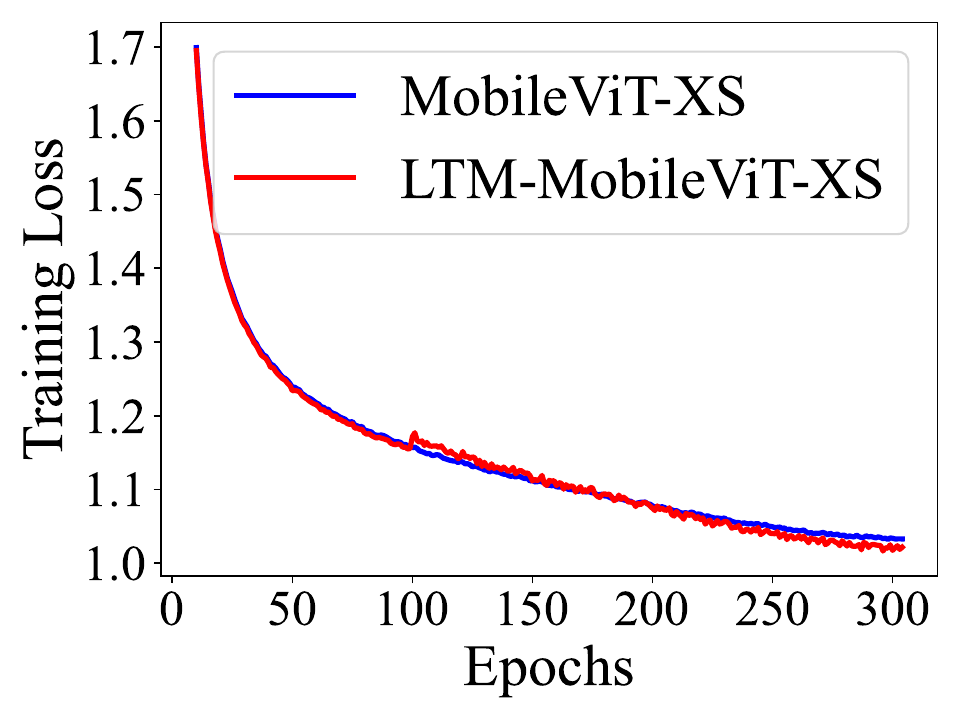}
        \vspace{-2mm}
        \caption{Training loss comparison between MobileViT-XS and LTM-MobileViT-XS.}
    \end{subfigure}
    \hspace{5mm}
    \begin{subfigure}[b]{0.225\textwidth}
        \centering
        \includegraphics[height=0.7\columnwidth]{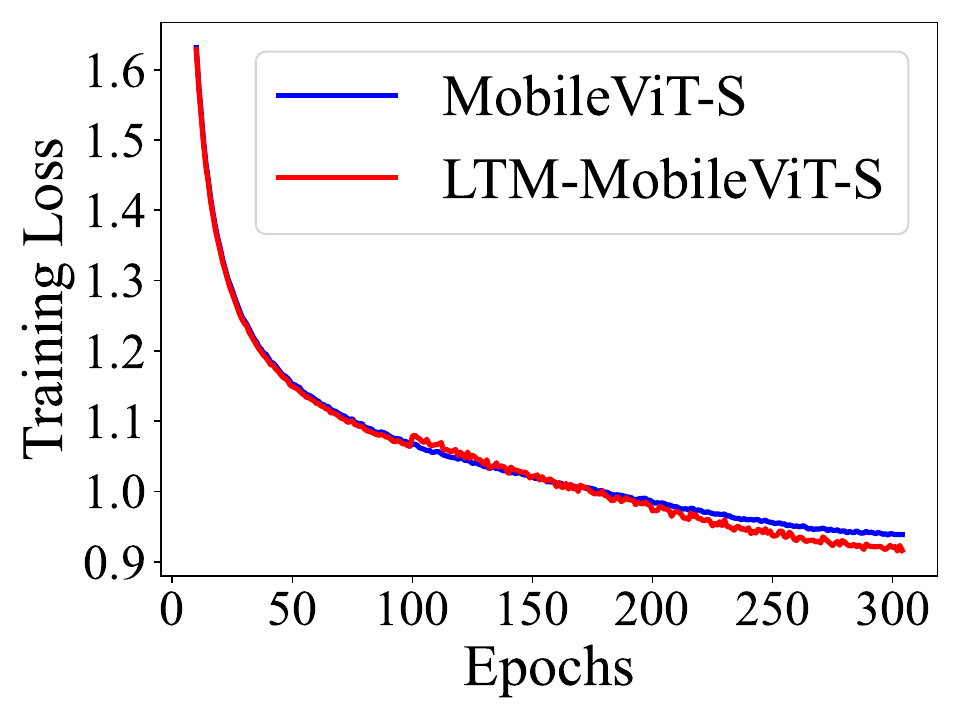}
        \vspace{-2mm}
        \caption{Training loss comparison between MobileViT-S and LTM-MobileViT-S.}
    \end{subfigure}
    \hspace{5mm}
    \begin{subfigure}[b]{0.225\textwidth}
        \centering
        \includegraphics[height=0.7\columnwidth]{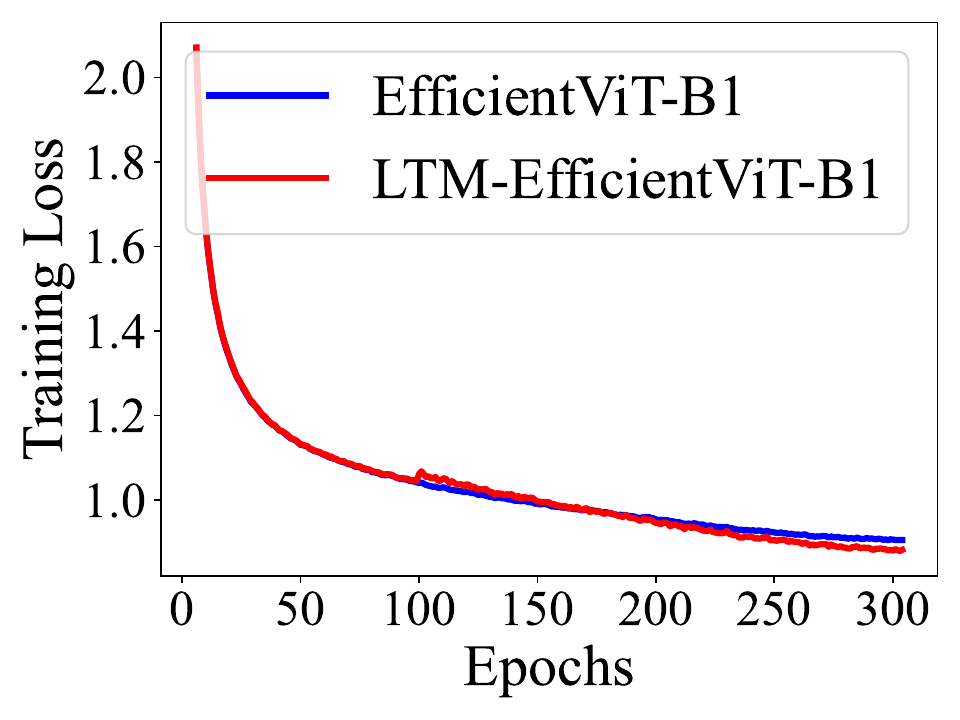}
        \vspace{-2mm}
        \caption{Training loss comparison between EfficientViT-B1 and LTM-EfficietViT-B1.}
    \end{subfigure}
        \hspace{5mm}
    \begin{subfigure}[b]{0.225\textwidth}
        \centering
        \includegraphics[height=0.7\columnwidth]{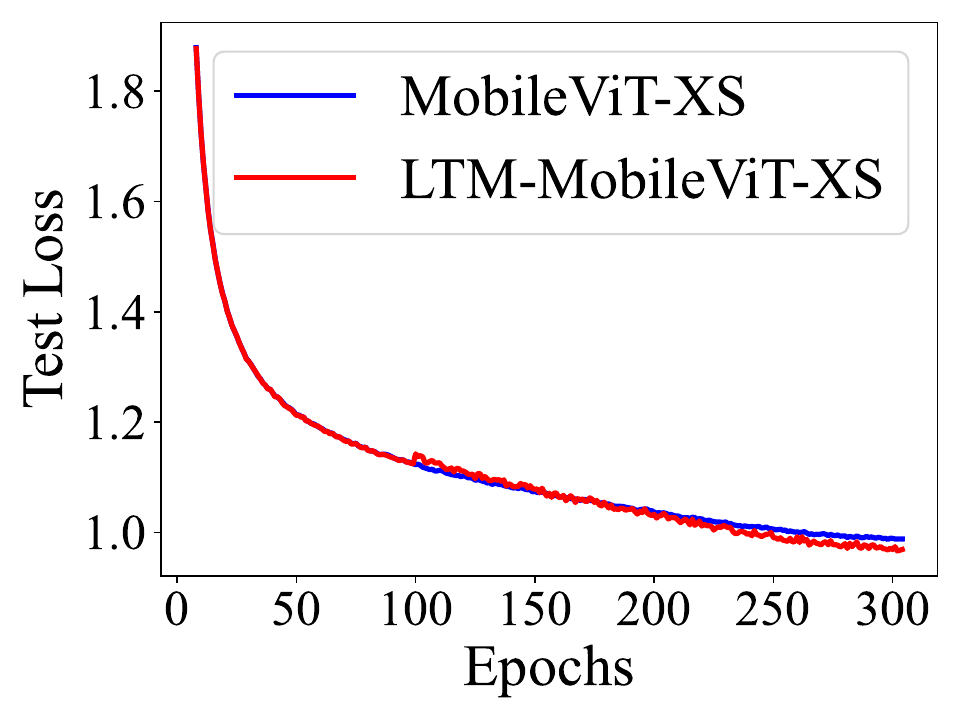}
        \vspace{-2mm}
        \caption{Test loss comparison between MobileViT-XS and LTM-MobileViT-XS.}
    \end{subfigure}
        \hspace{5mm}
    \begin{subfigure}[b]{0.225\textwidth}
        \centering
        \includegraphics[height=0.7\columnwidth]{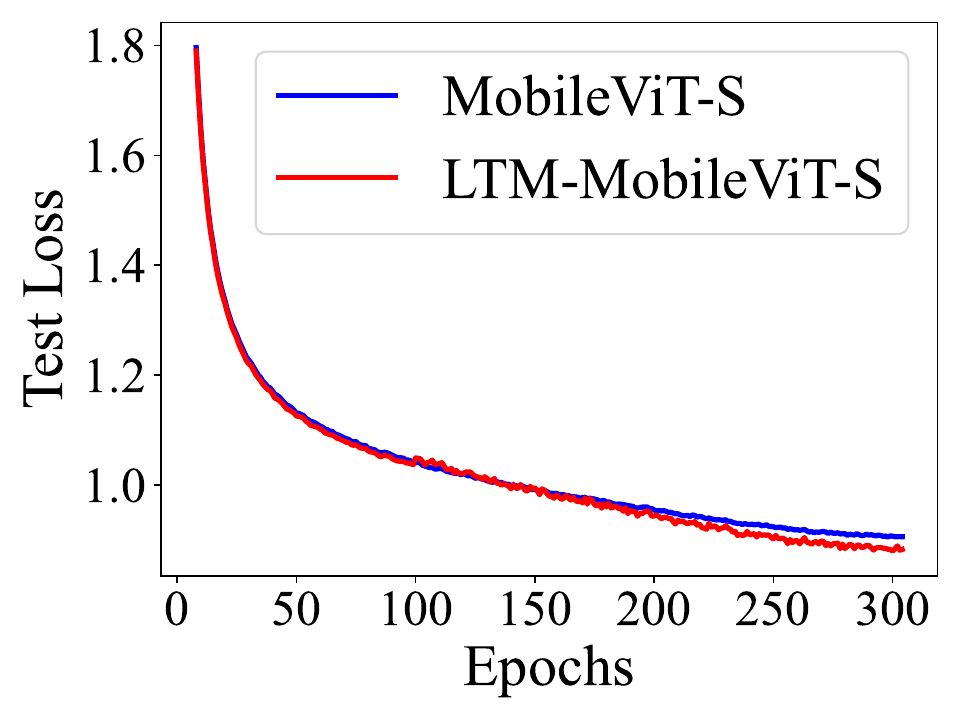}
        \vspace{-2mm}
        \caption{Test loss comparison between MobileViT-S and LTM-MobileViT-S.}
    \end{subfigure}
    \hspace{5mm}
    \begin{subfigure}[b]{0.225\textwidth}
        \centering
        \includegraphics[height=0.7\columnwidth]{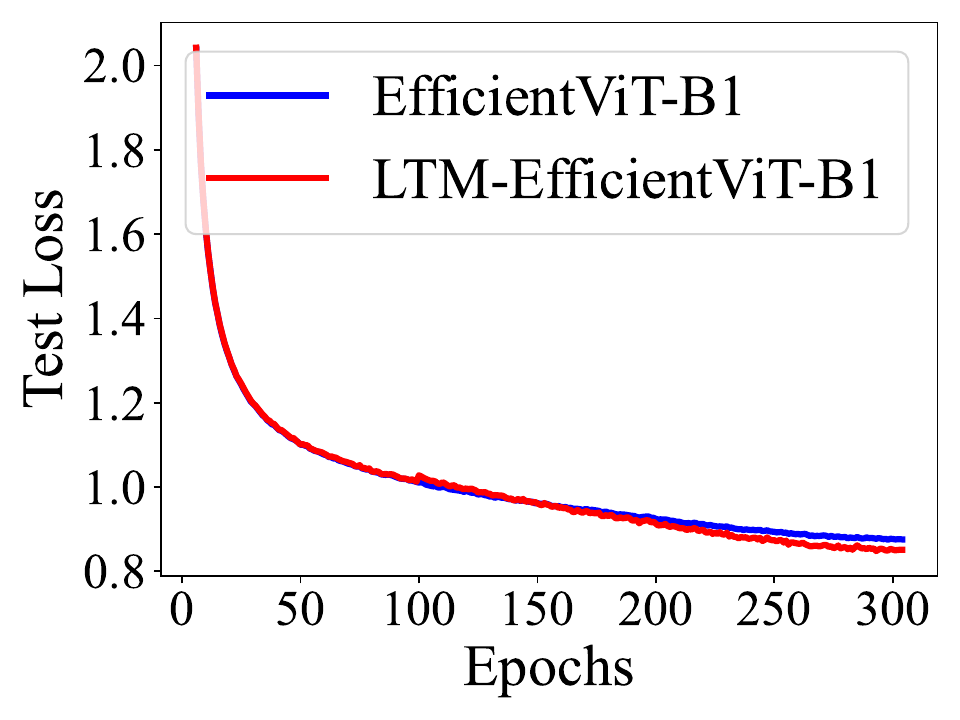}
        \vspace{-2mm}
        \caption{Test loss comparison between EfficientViT-B1 and LTM-EfficietViT-B1.}
    \end{subfigure}
\end{center}
\vspace{-5mm}
\caption{Training loss and test loss comparison between LTM-Transformer networks and corresponding baseline models.}
\label{fig:loss}
\end{figure*}
\subsubsection{Study on the IB loss and IB bound at Different Layers}
\label{sec:IB_different_layers}
To study how the IB loss $\textup{IB}(G)$, and the variational upper bound for the IB loss, $\textup{IBB}(G)$, decrease with respect to layer index $\ell$ of a LTM-Transformer network, $\textup{IB}(G)$ and $\textup{IBB}(G)$ for both MobileViT-S and LTM-MobileViT-S across different transformer layers are illustrated in Figure~\ref{fig:IB_plots}. Both models contain 9 transformer layers. It is observed from Figure~\ref{fig:IB_plots} that both $\textup{IB}(G)$ and $\textup{IBB}(G)$ decrease in deeper layers with larger layer indices of MobileViT-S and LTM-MobileViT-S.  This observation suggests that features in deeper layers correlate more closely with the class labels and less with the input features, adhering to the IB principle. Moreover,  LTM-MobileViT-S reduces both $\textup{IB}(G)$ and $\textup{IBB}(G)$ to lower levels in deeper layers compared to MobileViT-S. These observations evidence that the  mask module in the LTM-Transformer block, which generates the informative token merging mask by (\ref{eq:Gmask-GD-MLP-SGD}) can
effectively reduce both $\textup{IB}(G)$ and $\textup{IBB}(G)$,
better adhering to the IB principle than the vanilla
MobileViT-S.
At the early stage of the training after $100$ epochs, the IB bound and the IB loss of LTM-MobileViT-S are similar to those of the MobileViT-S.
After training for $300$ epochs, the IB bound and the IB loss of LTM-MobileViT-S are much smaller than those of the MobileViT-S.
\begin{figure*}[!htbp]
\begin{center}
\includegraphics[width=0.7\textwidth]{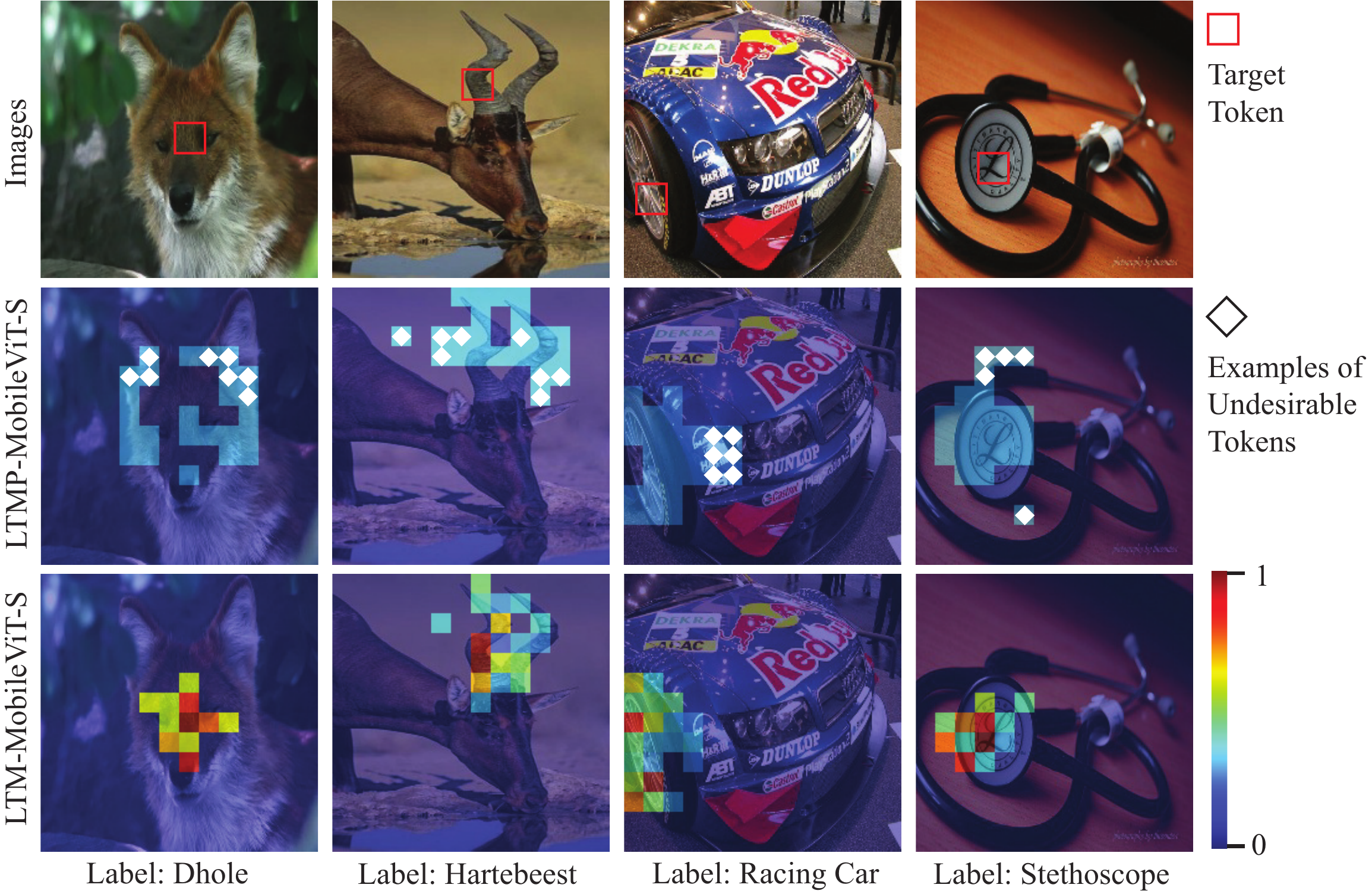}
\vspace{-3mm}
\caption{Comparisons of token merging weights generated in the first transformer block in LTM-MobileViT-S and LTMP-MobileViT-S. The target tokens resulting from the token merging process are marked with red boxes in the first row.
The token merging weights for tokens merged into the target tokens by LTMP-MobileViT-S are illustrated in the second row.
The token merging weights for tokens merged into the target tokens by LTM-MobileViT-S are illustrated in the third row.
Tokens that are more semantically similar to the target tokens should receive higher merging weights so as to achieve informative token merging.
Examples of the undesirable tokens, which are not semantically similar to the target tokens, are marked with white diamonds. In contrast with LTMP, the heatmaps of the merging weights generated by LTM illustrate that LTM usually assigns larger weights to tokens that are more semantically similar to the target tokens. }
\label{fig:merge-weights}
\end{center}
\vspace{-5mm}
\end{figure*}
\subsubsection{Training Loss and Test Loss of LTM-Transformers}
\label{sec:train-test-loss-LTM}

In this section, we illustrate the training loss and the test loss of LTM-MobileViT-XS, LTM-MobileViT-S, and LTM-EfficientViT-B1. In comparison, we also illustrate the training loss and test loss of MobileViT-XS, MobileViT-S, and EfficientViT-B1. All the models are trained for 300 epochs for each cycle with two cycles. The plots are illustrated in Figure~\ref{fig:loss} for the second cycle. It can be observed that LTM-Transformer networks achieve lower training losses and test losses at the end of the training, which demonstrates the benefit of LTM-Transformers in improving the performance of the visual transformers through the IB-inspired token merging.

\subsubsection{Visualization Results}
\label{sec:visual_results}
To study the effectiveness of LTM-Transformer in selecting informative tokens during the token merging process, we visualize the token merging masks in the first transformer block of LTMP-MobileViT-S and LTM-MobileViT-S for particular target tokens in selected images from ImageNet in Figure~\ref{fig:merge-weights}. Each image is divided into $16\times16$ tokens. For each example, we select only the most representative target token that encapsulates the critical features of the objects in the image, and the target token is a weighted average of several self-attention output tokens with the token merging weights
in the token merging mask.
The input images are illustrated in the first row.
Every target token is marked in red boxes in the first row, which is
calculated by $\tilde{X}(G^{(\ell)}) = \pth{Z^{\top}G^{(\ell)}}^{\top}$ as described in Section~\ref{sec:channel-selection-attention-weights}. The third row illustrates the token merging mask $G^{(\ell)}$ generated by the LTM for the target tokens with heatmaps.
We remark that the token merging masks for LTM and LTMP contain the token merging weights for these two methods.
The token merging masks generated by LTMP for the same target tokens are illustrated in the second row.
The class labels for each image are presented at the bottom of each column.
Tokens that are more semantically similar to the target token are expected to receive larger token merging weights in the token merging masks, so that the original tokens can contribute to the target token in accordance to their informativeness with respect to the target token.
The visualization results illustrate that the mask module in the LTM-Transformer block usually assigns higher merging weights to tokens that are more semantically similar to the target token.
In contrast, LTMP assigns uniform merging weights to all the selected tokens, which usually contain some undesirable tokens that are not semantically similar or even relevant to the target tokens.
In the example of the dhole in the first column, the LTM-Transformer block assigns larger merging weights to tokens around the eyes of the hole as the target token is also around the eyes.
In contrast, the competing baseline LTMP mistakenly merges tokens on the ears of the hole into the target token.
In the example of the hartebeest in the second column, the LTM-Transformer block assigns larger weights to tokens on the twisted horns of the hartebeest as the target token is also on the twisted horns of the hartebeest. The uneven distribution of token merging weights computed by LTM is consistent with the goal of informative token merging in our LTM-Transformer, where more informative original tokens, whose features are more similar to that of the target token, should receive higher token merging weights than other original tokens. In contrast, the competing baseline LTMP mistakenly merges tokens in the background around the twisted horns of the hartebeest.
In the example of the racing car in the third column, the LTM-Transformer block assigns larger weights to tokens on the wheel of the car as the target token is also on the wheel. In contrast, the competing baseline LTMP mistakenly merges tokens around the light of the car. In the example of the Stethoscope in the fourth column, the LTM-Transformer block assigns larger weights to tokens on the diaphragm of the stethoscope. In contrast, the competing baseline LTMP mistakenly merges tokens on the top of the diaphragm of the stethoscope in the background.

In addition, the visualization results illustrate that LTMP usually assigns larger weights to tokens that are more semantically similar to the target tokens. For instance, in the example of the hartebeest, tokens on the left horn of the hartebeest have larger token merging weights than tokens on the right horn of the hartebeest because the target token is on the left horn, which exhibits visual patterns that are not shared on the right horn. Similarly, in the example of the stethoscope, tokens around the center of the diaphragm have larger merging weights than tokens on the boundary of the diaphragm because the target token is in the center of the diaphragm, which exhibits visual patterns that are not shared on the boundary of the diaphragm.

\subsubsection{Study on the Impacts of Compression Ratio at Different Layers of the LTM-Transformer}
\label{sec:layer_compression_ratio}

In our experiments, the compression ratio $r$ for all the layers in the LTM-Transformer is set to a fixed value of $0.7$. To study how the compression ratio at different layers in the LTM-Transformer influences the IB loss at the corresponding layer, we train ablation models by changing the compression ratios at different layers from $0.7$ to $0.3$, $0.5$, and $0.9$. For each of the ablation models, we only change the compression ratio at a specific layer and keep the compression ratios for other layers fixed at $0.7$. The ablation study is performed for all $9$ layers of the LTM-MobileViT-S in the fine-tuning setup, and all ablation models are fine-tuned for $50$ epochs. Let $\ell$ denote the layer index where the compression ratio is changed. The IB loss at the corresponding layer where the compression ratio is changed and the top-1 accuracy of the ablation models are shown in Table~\ref{tab:layer_compression_ratio}. It is observed that reducing the compression ratio from $r=0.9$ to $r=0.7$ can reduce the IB loss at the corresponding layer and increase the top-1 accuracy, benefiting from merging redundant tokens. However, further reducing the compression ratio to $r=0.5$ and $r=0.3$ leads to an increased IB loss and reduced top-1 accuracy, especially at bottom layers where $\ell \leq 5$. This is attributed to the excessive merging of tokens, potentially resulting in the loss of informative features that are critical for the classification task.

\begin{table}[!htbp]
\centering
\resizebox{1\columnwidth}{!}{%
\begin{tabular}{c|cc|cc|cc|cc}
\hline
Layer & \multicolumn{2}{c|}{$r=0.3$} & \multicolumn{2}{c|}{$r=0.5$} & \multicolumn{2}{c|}{$r=0.7$} & \multicolumn{2}{c}{$r=0.9$} \\ \cline{2-9}
Index $\ell$ & Top-1       & IB Loss       & Top-1       & IB Loss       & Top-1       & IB Loss       & Top-1       & IB Loss       \\ \hline
1     & 78.22       &  0.00214      & 78.70       &  0.00190      & 79.05       &  0.00178      & 79.02       &  0.00185      \\
2     & 78.43       & -0.00137      & 78.85       & -0.00140      & 79.05       & -0.00156      & 79.01       & -0.00148      \\
3     & 78.55       & -0.00622      & 78.90       & -0.00630      & 79.05       & -0.00648      & 78.98       & -0.00641      \\
4     & 78.57       & -0.00725      & 79.05       & -0.00742      & 79.02       & -0.00762      & 78.98       & -0.00750      \\
5     & 78.75       & -0.00938      & 78.83       & -0.00950      & 79.05       & -0.01005      & 79.00       & -0.00977      \\
6     & 78.95       & -0.01185      & 79.02       & -0.01197      & 79.05       & -0.01204      & 79.04       & -0.01199      \\
7     & 79.00       & -0.01430      & 79.04       & -0.01464      & 79.05       & -0.01488      & 79.01       & -0.01486      \\
8     & 79.00       & -0.01525      & 79.03       & -0.01556      & 79.05       & -0.01592      & 79.05       & -0.01592      \\
9     & 79.04       & -0.01718      & 79.03       & -0.01704      & 79.05       & -0.01725      & 79.00       & -0.01728      \\ \hline
\end{tabular}
}
\vspace{-2mm}
\caption{Ablation study on the impacts of compression ratio in reducing the IB loss at different layers of the LTM-MobileViT-S. All ablation models are fine-tuned for 50 epochs. $\ell$ denotes the layer index where the compression ratio is changed. The IB loss at the $\ell$-th layer and the top-1 accuracy of the ablation models are shown.}
\label{tab:layer_compression_ratio}
\vspace{-2mm}
\end{table}

In addition, we have performed an ablation study on evaluating the transfer learning capability of the LTM-Transformer in Section 2.1 of the supplementary. The results show that LTM does not affect the transfer learning capability of the original model.
Moreover, we have shown that the parameter-efficient post-training method, LoRA~\cite{hu2022lora}, can significantly reduce the training costs of the LTM-Transformer without scarfing the performance in Section 2.2 of the supplementary.
In Section 2.3 of the supplementary, we have shown that the LTM-Transformer exhibits faster inference speed than the models compressed by competing token merging methods across different batch sizes.
In Section 3 of the supplementary, we have discussed the potential future work of this paper for improved training efficiency.

\vspace{-2mm}
\section{Conclusion}
\label{sec:conclusion}
\vspace{-.03in}
In this paper, we propose a novel transformer block, Transformer with Learnable Token Merging, or LTM-Transformer. LTM-Transformer blocks perform token
merging so as to render a transformer network
with less FLOPs and faster inference speed.
A LTM-Transformer block generates an informative token merging mask
for token merging in a learnable manner, which is
inspired by the reduction of the
Information-Bottleneck loss.
A network with LTM-Transformer blocks can either be trained from scratch or fine-tuned from the pre-trained backbone, and it
enjoys a reduction of IB loss and reduced FLOPs while maintaining
a compelling prediction accuracy. We demonstrate the effectiveness of
LTM-Transformer by replacing all the transformer blocks in several popular visual transformers with LTM-Transformer blocks. Extensive experiments on various tasks demonstrate the effectiveness of the LTM-Transformer.

\ifCLASSOPTIONcaptionsoff
  \newpage
\fi



\bibliographystyle{IEEEtran}
\bibliography{ref}
%
%
%

%
\vspace{-8mm}

 \begin{IEEEbiography}[{\includegraphics[width=1in,height=1.25in,clip,keepaspectratio]{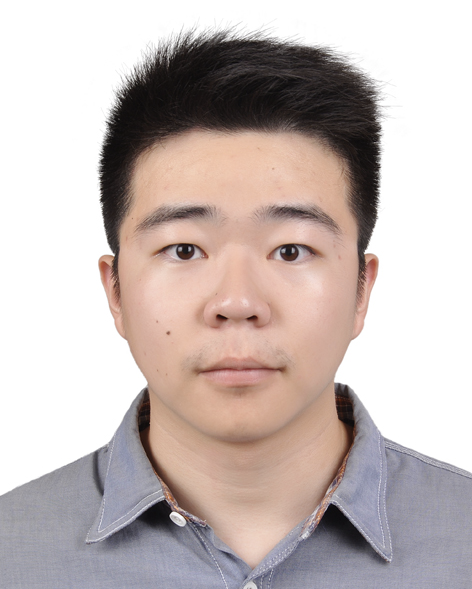}}]{Yancheng Wang}
 is a Ph.D. student in Computer Science at Arizona State University. He received his B.S. degree in Mathematics and Applied Mathematics from Xi'an Jiaotong University in 2018. His research interest is in developing robust and efficient deep learning methods, such as neural architecture search, and their applications to image and graph data.
 \end{IEEEbiography}
 \vspace{-8mm}
 \begin{IEEEbiography}[{\includegraphics[width=1in,height=1.25in,clip,keepaspectratio]{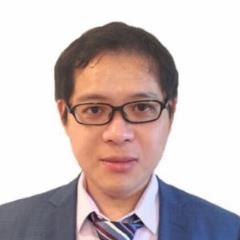}}]{Yingzhen Yang}
 Dr. Yingzhen Yang has been an assistant professor at SCAI since 2019, and his research areas include deep learning and statistical machine learning with more than $70$ papers published in these areas. He is the recipient of the Best Paper Finalist for ECCV (European Conference on Computer Vision) in 2016. He has served as a senior program committee member, a committee member or a reviewer for premier conferences and journals in machine learning, deep learning, and artificial intelligence, including Journal of Machine Learning Research (JMLR), Journal of Artificial Intelligence Research (JAIR), IEEE Transactions on Image Processing (TIP), the International Joint Conferences on Artificial Intelligence (IJCAI), the Association for the Advancement of Artificial Intelligence (AAAI) Conference on Artificial Intelligence, the International Conference on Machine Learning (ICML), the Annual Conference on Neural Information Processing Systems (NeurIPS), the International Conference on Learning Representations (ICLR).
 \end{IEEEbiography}

\end{document}